\documentclass[letterpaper,11pt]{article}
\usepackage[title]{appendix}
\usepackage{subcaption}

\newcommand{\BigO}[1]{\ensuremath{\operatorname{O}\bigl(#1\bigr)}}

\RequirePackage{algorithm}
\RequirePackage[noend]{algpseudocode}
\usepackage{etoolbox}
\RequirePackage{amssymb}

\RequirePackage{fullpage}

\usepackage{comment}
\RequirePackage[colorlinks]{hyperref}
\RequirePackage{amsthm,amsmath}
\RequirePackage[numbers]{natbib}
\usepackage{mathtools}
\newtheorem{theorem}{Theorem}

\newtheorem{corollary}[theorem]{Corollary}

\newtheorem{lemma}[theorem]{Lemma}

\newtheorem{prop}{Proposition}
\allowdisplaybreaks

\newcommand{\name}{GosInE}

\makeatletter

\newcommand*{\algrule}[1][\algorithmicindent]{%
  \makebox[#1][l]{%
    \hspace*{.2em}
    \vrule height .75\baselineskip depth .25\baselineskip
  }
}

\newcount\ALG@printindent@tempcnta
\def\ALG@printindent{%
    \ifnum \theALG@nested>0
    \ifx\ALG@text\ALG@x@notext
    \else
    \unskip
    \ALG@printindent@tempcnta=1
    \loop
    \algrule[\csname ALG@ind@\the\ALG@printindent@tempcnta\endcsname]%
    \advance \ALG@printindent@tempcnta 1
    \ifnum \ALG@printindent@tempcnta<\numexpr\theALG@nested+1\relax
    \repeat
    \fi
    \fi
}
\patchcmd{\ALG@doentity}{\noindent\hskip\ALG@tlm}{\ALG@printindent}{}{\errmessage{failed to patch}}
\patchcmd{\ALG@doentity}{\item[]\nointerlineskip}{}{}{} 
\makeatother

\title{The Gossiping Insert-Eliminate Algorithm for Multi-Agent Bandits}
\author{Ronshee Chawla\thanks{Equal Contribution}\footnote{Electrical and Computer Engineering, University of Texas at Austin. email: ronsheechawla@utexas.edu}, Abishek Sankararaman\footnotemark[1]\footnote{Electrical and Computer Engineering, University of Texas at Austin. email: abishek@utexas.edu}, Ayalvadi Ganesh\footnote{Department of Mathematics, University of Bristol. email: a.ganesh@bristol.ac.uk} and Sanjay Shakkottai\footnote{Electrical and Computer Engineering, University of Texas at Austin. email: sanjay.shakkottai@utexas.edu}}
\date{}

\begin{document}

%

%

\maketitle
\begin{abstract} 
    
    We consider a decentralized multi-agent Multi Armed Bandit (MAB) setup consisting of $N$ agents, solving the same MAB instance to minimize individual cumulative regret. In our model, agents collaborate by exchanging messages through pairwise gossip style communications on an arbitrary connected graph. 
    We develop two novel algorithms, where each agent only plays from a subset of all the arms. Agents use the communication medium to recommend only arm-IDs (not samples), and thus update the set of arms from which they play. We establish that, if agents communicate $\Omega(\log(T))$ times through any connected pairwise gossip mechanism, then every agent's regret is a factor of order $N$ smaller compared to the case of no collaborations.  Furthermore, we show that the communication constraints only have a second order effect on the regret of our algorithm. We then analyze this second order term of the regret to derive bounds on the regret-communication tradeoffs. Finally, we empirically evaluate our algorithm and conclude that the insights are fundamental and not artifacts of our bounds. 
    We also show a lower bound which gives that the regret scaling obtained by our algorithm cannot be improved even in the absence of any communication constraints.
    Our results demonstrate that even a minimal level of collaboration among agents greatly reduces regret for all agents.  
    


\end{abstract}

\section{Introduction}
 Multi Armed Bandit (MAB) is a classical model (\cite{lattimore_book},\cite{bubeck_book}), that captures the explore-exploit trade-off in making online decisions. MAB paradigms have found  applications in many large scale systems such as ranking on search engines \cite{bandit_search}, displaying advertisements on e-commerce web-sites \cite{bandit_commerce}, model selection for classification \cite{hyperband} and real-time operation of wireless networks \cite{bandit_wireless}. Oftentimes in these settings, the decision making is distributed among many agents. For example, in the context of web-servers serving either search ranking or placing advertisements, due to the the volume and rate of user requests, multiple servers are deployed to perform the same task \cite{delay_nonstochastic}. Each server, makes decisions (which can be modeled as a MAB \cite{bandit_search}) on rankings or placing advertisements and also collaborate with other servers by communicating over a network \cite{delay_nonstochastic}. In this paper, we study a multi-agent MAB model in which agents collaborate to reduce individual cumulative regret.\\
 

 \noindent {\bf Model Overview} - Our model generalizes the problem setting described in \cite{sigmetrics}. Concretely, our model consists of $N$ agents, each playing the same instance of a $K$ armed stochastic MAB, to minimize its cumulative regret. At each time, every agent pulls an arm and receives a stochastic reward independent of everything else (including other agents choosing the same arm at the same time). Additionally, an agent can choose after an arm pull, to receive a message from another agent through an \emph{information pull}. Agents have a \emph{communication budget}, which limits how many times an agent can pull information. If any agent $i \in \{1,\cdots,N\}$ chooses to receive a message through an information-pull, then it will contact another agent $j$ chosen independent of everything else, at random from a { distribution} $P(i,\cdot)$ (unknown to the agents) over $\{1,\cdots,N\}$. The agents thus cannot actively choose from whom they can receive information, rather they receive from another randomly chosen agent. The $N \times N$ matrix $P$ with its $i^{\rm{th}}$ row being the distribution $P(i,\cdot)$ is denoted as the {gossip matrix}. Agents take actions (arm-pulls, information-pulls and messages sent) only as a function of their past history of arm-pulls, rewards and received messages from information-pulls and is hence \emph{decentralized}. \\

\noindent {\color{black} \textbf{Model Motivations} - The problem formulation and the communication constraints aim to capture key features of many settings involving multiple agents making distributed decisions. We highlight two examples in which our model is applicable. The first example is a setting consisting of $N$ computer servers (or agents), each handling requests for web searches from different users on the internet  \cite{buccapatnam,langford-news}. For each keyword, one out of a set of M ad-words needs to be displayed, which can be viewed as choosing an arm of a MAB. Here, each server is making decisions on which ad to display (for the chosen keyword) independently of other servers. Further, the rewards obtained by different servers are independent because the search users are different at different servers. The servers can also communicate with each other over a network in order to collaborate to maximize revenue (i.e., minimize cumulative regret).
 
 A second example is that of collaborative recommendation systems, e.g., where multiple agents (users) in a social network are jointly exploring restaurants in a city \cite{sigmetrics}. The users correspond to agents, and each restaurant can be modeled as an arm of a MAB providing stochastic feedback. The users can communicate with each other over a social network, personal contact or a messaging platform to receive recommendation of restaurants (arms) from others to minimize their cumulative regret, where regret corresponds to the loss in utility incurred by each user per restaurant visit. Furthermore, if the restaurants/customers can be categorized into a finite set of contexts (say, e.g. by price: low-cost/mid-price/high-end, type of cuisine: italian, asian, etc.), our model is applicable per context.\\}

\noindent {\bf Key Contributions}:
\\

\noindent\textbf{1. Gossiping Insert-Eliminate (\name) Algorithm} -  In our algorithms (Algorithm \ref{algo:main-algo} and \ref{algo:main-algo-prob}), agents only choose to play from among a small subset (of cardinality $\lceil \frac{K}{N} \rceil +2$) of arms at each time. Agents in our algorithm accept the communication budget as an input and use the communication medium to \emph{recommend} arms, i.e., agents communicate the arm-ID of their current estimated best arm. Specifically, agents do not exchange samples, but only recommend an arm index. On receiving a recommendation, an agent updates the set of arms to play from: it discards its estimated worst arm in its current set and replaces it by the recommended new arm.

Thus, our algorithm is \emph{non monotone} with respect to the set of arms an agent plays from, as agents can discard an arm in a phase and then subsequently bring the arm back and play it in a later phase, if this previously discarded arm gets recommended by another agent. This is in contrast to most other bandit algorithms in the literature. On one hand, classical regret minimization algorithms such as UCB-$\alpha$ \cite{ucb_auer} or Thompson sampling \cite{thompson} allow sampling from any arm at all points in time (no arm ever discarded). On the other hand, pure explore algorithms such as successive rejects \cite{successive-rejectst} are monotone with respect to the arms, i.e., a discarded arm is never subsequently played again. The social learning algorithm in \cite{sigmetrics} is also monotone, as the subset of arms from which an agent plays at any time is non-decreasing. 
In contrast, in this paper we show that even if an agent (erroneously) discards the best arm from its playing set, the recommendations ensure that with probability $1$, the best arm is eventually back in the playing set.\\

 \noindent   \textbf{2. Regret of \name \ Algorithm} - Despite agents playing among a {\em time-varying} set of arms of cardinality $\lceil \frac{K}{N} \rceil +2$, we show that the regret of any agent is (Theorems \ref{thm:strong_result} and \ref{thm:prob_main_result}) $\BigO{ \left( \frac{\lceil \frac{K}{N} \rceil +1}{\Delta_2} \log(T) \right)} + C $. Here, $\Delta_2$ is the difference in the mean rewards of the best and second best arm and $C$ is a constant depending on communication constraints and independent of time. We show that the regret scaling holds \emph{for any connected gossip matrix $P$ and communication budget scaling as $\Omega(\log(T))$}. Thus, any agent's asymptotic regret is independent of the the gossip matrix $P$ or the communication budget (Corollary \ref{cor-com-insensitive}). If agents never collaborate (communication budget of $0$), the system is identical to each agent playing a standard $K$ arm MAB, in which case the regret scales as $\BigO{\frac{K}{\Delta_2}\log(T)}$ \cite{lai_robbins},\cite{ucb_auer}. Thus, our algorithms reduce the regret of any agent by a factor of order $N$ from the case of no collaborations. Furthermore, a lower bound in Theorem \ref{thm:full_interaction} (and the discussion in Section~\ref{sec:tradeoffs}) shows that this scaling with respect to $K$ and $N$ cannot be improved by any algorithm, communication budget or gossip matrix. Specifically, we show that even if an agent has knowledge of the entire system history of arms pulled and rewards obtained by other agents, the regret incurred by every agent is only a factor of order $N$ smaller than the case of no collaborations. Moreover, our regret scaling significantly improves over that of \cite{sigmetrics}, which applies only to the complete graph among agents, in which the regret scales as $\BigO{\frac{\lceil \frac{K}{N}\rceil + \log(N)}{\Delta_2} \log(T)}$. Thus, despite communication constraints, our algorithm leverages collaboration effectively.
    \\

\noindent\textbf{3. Communication/Regret Trade-Off} - The second order constant term in our regret bound captures the trade-off between communications and regret.  As an example, we show in Corollary \ref{cor:regret-comm-main-paper} that, if the communication budgets scale polynomially, i.e., agents can pull information at-most $t^{1/\beta}$ times over a time horizon of $t$, for some $\beta > 1$, when the agents are connected by a ring graph (the graph with poorest connectivity), the constant term in the regret scales as $(N)^{\beta}$ (upto poly-logarithmic factor), whereas the regret scales as $(\log(N))^{\beta}$, in the case when agents are connected by the complete graph.
Thus, we see that there is an exponential improvement (in the additive constant) in the regret incurred, when changing the network among agents from the ring graph to the complete graph. 
In general, we give through an explicit formula (in Corollary \ref{cor:regret-comm-main-paper}) that, if the gossip matrix $P$ has smaller conductance (i.e., a poorly connected network), then the regret incurred by any agent is higher. Similarly, we also establish the fact that if the communication budget per agent is higher, then the regret incurred is lower (Corollary \ref{cor:reg-comm-tradeoff}). We further conduct numerical studies that establish these are fundamental and not artifacts of our bounds.

\section{Problem Setup}
\label{sec:model}

Our model generalizes the setting in \cite{sigmetrics}. In particular, our model, imposes communication budgets and allows for general gossip matrices $P$, while the model in \cite{sigmetrics} considered only the complete graph among agents.
\\

\noindent {\bf Arms of the MAB} - We consider $N$ agents, each playing the same instance of a $K$ armed stochastic MAB to minimize cumulative regret. The $K$ arms have unknown average rewards denoted by $\mu_1,\cdots,\mu_K$, where for every $i \in \{1,\cdots,K\}$, $\mu_i \in (0,1)$. Without loss of generality, we assume $1 > \mu_1 > \mu_2 \geq \mu_3 \cdots \geq \mu_K \geq 0$. However, the agents are not aware of this ordering. For all $j \in \{2,\cdots,K\}$, denote by $\Delta_j := \mu_1 - \mu_{j}$. The assumption on the arm-means imply that $\Delta_j > 0$, for all $j \in \{2,\cdots,K\}$. 
\\

\noindent {\bf Network among Agents} - We suppose that the agents are connected by a network denoted by a $N \times N$ \emph{gossip matrix} $P$, where for each $i \in \{1,\cdots,N\}$, the $i^{\text{th}}$ row $P(i,\cdot)$ is a probability distribution over $\{1,\cdots,N\}$. This matrix is fixed and unknown to the agents.
\\

\noindent {\bf Agent Actions} - We assume that time is slotted (discrete), with each time slot divided into an arm-pulling phase followed by an information-pulling phase. In the arm-pulling phase, all agents pull one of the $K$ arms and observe a stochastic Bernoulli reward, independent of everything else. In the information pulling phase, if an agent has communication budget, it can decide to receive a message from another agent through an information pull. A non-negative and non-decreasing sequence $(B_t)_{t \in \mathbb{N}}$ specifies the communication budget, where no agent can pull information for more than $B_t$ times in the first $t$ time slots for all $t\geq 0$. If any agent $i \in \{1,\cdots,N\}$, chooses to pull information in the information-pulling phase of any time slot, it will contact another agent $j \in \{1,\cdots,N\}$ chosen independently of everything else, according to the probability distribution given by $P(i,\cdot)$. Thus, agents receive information from a randomly chosen agent according to a \emph{fixed} distribution, rather than actively choosing the agents based on observed samples. When any agent $j$ is contacted by another agent in the information pulling phase of a time-slot, agent $j$ can communicate a limited ($\BigO{\log(NK)}$ number of bits. Crucially, the message length does not depend on the arm-means or on the time index. 
\\

\noindent {\bf Decentralized System} - Each action of an agent, i.e., its arm pull, decision to engage in an information pull and the message to send when requested by another agent's information pull, can only depend on the agent's past history of arms pulled, rewards obtained and messages received from information pulls. We allow each agent's actions in the information pulling phase (such as whether to pull information and what message to communicate if asked for), to depend on the agent's outcome in the arm-pulling phase of that time slot.
\\

\noindent {\bf Performance Metric} - Each agent minimizes their expected cumulative regret. For an agent $i \in \{1,\cdots,N\}$ and time $t \in \mathbb{N}$, denote by $I_t^{(i)} \in \{1,\cdots,K\}$ to be the arm pulled by agent $i$ in the arm-pulling phase of time slot $t$. The regret of agent $i \in \mathbb{N}$, after $T$ time slots (arm-pulls) is defined as ${R}_{T}^{(i)} :=  \sum_{t=1}^{T}(\mu_1 - \mu_{I_t^{(i)}})$ and the expected cumulative regret is $\mathbb{E}[R_T^{(i)}]$\footnote{Expectation is with respect to all randomness, i.e., rewards, communications and possibly the algorithm.}.



\section{Synchronous \name \ Algorithm}
\label{sec:synch_algo}


We describe the algorithm by fixing an agent $i \in \{1,\cdots,N\}$.
\\



\noindent {\bf Input Parameters} - The algorithm has three inputs {\em (i)} a {communication budget} $(B_t)_{t \in \mathbb{N}}$, {\em (ii)} $\alpha > 0$ and {\em (iii)} $\varepsilon > 0$. From this communication budget, we construct a sequence $(A_x)_{x \in \mathbb{N}}$ such that 
\begin{equation}\label{eqn:A_x_defn}
    A_x = \max \left( \min \{t \in \mathbb{N}, B_t \geq x\}, \lceil (1+x)^{1+\varepsilon}\rceil \right).
\end{equation}
Every agent, only pulls information in time slots $(A_x)_{x \in \mathbb{N}}$. This automatically respects the communication budget constraints.
Since agents engage in information pulling at common time slots, we term the algorithm, synchronous.
{\color{black} The parameter $\epsilon$ ensures that the time intervals between the instants when agents request for an arm are well separated. In particular, having $\varepsilon > 0$ ensures that the inter-communication times scale at least polynomially in time. As we shall see in the analysis, this only affects the regret scaling in the second order term.} 
\\

\noindent {\bf Initialization} - Associated with each agent $i \in \{1,\cdots,N\}$, is a \emph{sticky}\footnote{The choice of term \emph{sticky} is explained in the sequel.} set of arms - 
\begin{align}
\label{eqn:hat_S_i}
    \widehat{S}^{(i)} =  \bigg\{  \left( (i-1) \bigg\lceil \frac{K}{N} \bigg\rceil  \mod K \right)+1  ,\cdots, \left( i\bigg\lceil \frac{K}{N} \bigg\rceil -1 \mod K\right)+1\bigg\}.
\end{align}

 Notice that the cardinality $|\widehat{S}^{(i)}| = \lceil \frac{K}{N} \rceil$. In words, we are partitioning the total set of arms, into sets of size $\lceil \frac{K}{N} \rceil$ with the property that $\bigcup_{i=1}^{N} \widehat{S}^{(i)} = \{1,\cdots,K\}$. For instance, if $K = N$, then for all $i \in \{1,\cdots,N\}$, $\widehat{S}^{(i)} = \{i\}$. Denote by the set
 $U_0^{(i)} = \{i\lceil \frac{K}{N} \rceil \mod K \} $ and $L_0^{(i)} = \{i\lceil \frac{K}{N} \rceil +1\mod K\} $ and 
 \begin{align}
 \label{eqn:S_i_0}
     S_0^{(i)} = \widehat{S}^{(i)} \cup U_0^{(i)} \cup L_0^{(i)}.
 \end{align}

\noindent {\bf UCB within a phase} - The algorithm proceeds in phases with all agents starting in phase $0$. Each phase $j \geq 1$ lasts from time-slots $A_{j-1} + 1$ till time-slot $A_j$, both inclusive\footnote{We use the convention $A_{-1} = 0$}. We shall fix a phase $j \geq 0$ henceforth in the description. For any arm $l \in \{1,\cdots,K\}$ and any time $t \in \mathbb{N}$, $T_l^{(i)}(t)$ is the total number of times agent $i$ has pulled arm $l$, upto and including time $t$ and by $\widehat{\mu}_l^{(i)}(t) $, the empirical observed mean\footnote{ $\widehat{\mu}_l^{(i)}(t) =0 $ if $T_l^{(i)}(t)=0$}. Agent $i$ in phase $j$, chooses arms from $S_j^{(i)}$ according to the UCB-$\alpha$ policy of \cite{ucb_auer} where the arm is selected from 
$\arg\max_{l \in S_j^{(i)}} \left(\widehat{\mu}_l^{(i)}(t-1) + \sqrt{\frac{\alpha \ln(t)}{T_l^{(i)}(t-1)}} \right)$.
\\

\noindent {\bf Pull Information at the end of a phase} -  The message received (arm-ID in our algorithm) in the information-pulling phase of time slot $A_j$ is denoted by  $\mathcal{O}_j^{(i)} \in \{1,\cdots,K\}$. Every agent, when asked for a message in the information-pulling phase of time-slot $A_j$, will send the arm-ID it played the most in phase $j$.
\\


\noindent {\bf Update arms at the beginning of a phase} -  If $\mathcal{O}_j^{(i)} \in S_j^{(i)}$, then $S_{j+1}^{(i)} = S_{j}^{(i)}$. Else, agent $i$ discards the least played arm in phase $j$ from the set $S_j^{(i)} \setminus \widehat{S}^{(i)}$ and accepts the recommendation $\mathcal{O}_j^{(i)}$, to form the playing set $S_{j+1}^{(i)}$. Observe that the cardinality of $S_{j+1}^{(i)}$ remains unchanged. Moreover, the updating ensures that for all agents $i \in \{1,\cdots,N\}$ and all phases $j$, $\widehat{S}^{(i)} \subset S_j^{(i)}$, namely agents never drop arms from the set $\widehat{S}^{(i)}$. Hence, we term the set $\widehat{S}^{(i)}$, sticky.
\\

The pseudo-code of the Algorithm described above is given in Algorithm \ref{algo:main-algo}.
\begin{algorithm*}[t]
		\caption{Synch \name\; Algorithm (at Agent $i$)}
		\begin{algorithmic}[1]
				\State \textbf{Input }: Communication Budgets $(B_t)_{t \in \mathbb{N}}$ and UCB Parameter $\alpha$, $\varepsilon > 0$
				 
				\State \textbf{Initialization}:  $ \widehat{S}^{(i)},S_0^{(i)}$ according to Equations (\ref{eqn:hat_S_i}) and (\ref{eqn:S_i_0}) respectively.
				
				\State $j \gets 0$
				\State     $A_j = \max \left(\min \{t \geq 0, B_t \geq j\}, \lceil (1+j)^{1+\varepsilon}\rceil \right)$ \Comment{{Reparametrize the communication budget}}
				\For{Time $t \in \mathbb{N}$}
				\State Pull - $\arg\max_{l \in S_i^{(j)}} \left( 
				\widehat{\mu}_l^{(i)}(t-1) + \sqrt{\frac{\alpha \ln(t)}{T_l^{(i)}(t-1)}}\right)$
				\If {$t == A_j$} \Comment{{End of Phase}}
				\State $\mathcal{O}_j^{(i)} \gets$ {\ttfamily GetArm}($i,j$)
				\If {$\mathcal{O}_j^{(i)} \not \in S_i^{(j)}$}
				\State $U_{j+1}^{(i)} \gets \arg\max_{l \in \{U_j^{(i)},L_j^{(i)}\}}( T_l(A_j)- T_l(A_{j-1} ))$ \Comment{{The most played arm}}
				\State $L_{j+1}^{(i)} \gets \mathcal{O}_j^{(i)}$ 
				\State $S_{j+1}^{(i)} \gets \widehat{S}^{(i)} \cup L_{j+1}^{(i)} \cup U_{j+1}^{(i)}$ \Comment{{Update the set of playing arms}}
				\Else
				\State $S_{j+1}^{(i)} \gets S_j^{(i)}$.
				\EndIf 
				\State $j \gets j+1$
					\State  $A_j = \max \left(\min \{t \geq 0, B_t \geq j\}, \lceil (1+j)^{1+\varepsilon}\rceil \right)$ \Comment{{Reparametrize the communication budget}}
				\EndIf
				\EndFor
		\end{algorithmic}
	\label{algo:main-algo}
	\end{algorithm*}

\begin{algorithm}[t]
	\caption{Synchronous Arm Recommendation}
	\begin{algorithmic}[1]
		\Procedure {Getarm}{($i,j$)} \Comment{\emph{Input  an agent $i$ and Phase $j$}}
%
			\State $m \sim P(i,\cdot)$ \Comment{\emph{Sample another agent}}\\
		\Return $\arg\max_{l \in S_m^{(j)}} \left( T_l^{(m)}(A_j) - T_l^{(m)}(A_{j-1})\right)$ \Comment{\emph{Most Played arm in phase $j$ by agent $m$}}
		\EndProcedure
		\end{algorithmic}
	\end{algorithm}


	\subsection{Model Assumptions}

We make two mild assumptions on the inputs (a discussion is provided in Appendix  \ref{sec:assumption_discussion}).
\\

	\label{sec:assumptions}
\noindent {\bf (A.1)} The communication matrix $P$ is irreducible. Namely, for any two $i,j \in \{1,\cdots,N\}$, with $i\neq j$, there exists $2\leq l \leq N$ and $k_1,\cdots,k_l \in \{1,\cdots,N\}$, with $k_1 = i$ and $k_l=j$ such that the product $P(k_1,k_2)\cdots,P(k_{l-1},k_l) > 0$ is strictly positive.
\\


				\noindent {\bf (A.2)} The communication budget $(B_t)_{t \in \mathbb{N}}$ and $\varepsilon > 0$ is such that for all $D >0$, there exists $t_0(D)$ such that for all $t \geq t_0(D)$, $B_t \geq D \log(t)$ (i.e., $B_t = \Omega(\log(t))$). Furthermore, we shall assume a convexity condition, i.e.,  for every $x , y \in \mathbb{N}$ and $\lambda \in [0,1]$, $A_{\lfloor \lambda x + (1-\lambda)y \rfloor} \leq \lambda A_x + (1-\lambda)A_y$, where the sequence $(A_x)_{x \in \mathbb{N}}$ is given in Equation (\ref{eqn:A_x_defn}). Furthermore, $\sum_{ l \geq 2 } \frac{A_{2l}}{A_{l-1}^{3} } < \infty$.
				\\
				
{\color{black} \noindent For instance, if $B_t = \lceil t^{1/3} \rceil$, for all $t \geq 1$. and $\varepsilon < 2$, then $A_x = \lfloor x^3 \rfloor$, for all $x \geq 1$. Similarly, if $B_t = t$, for all $t \geq 1$, i.e., if the budget is adequate to communicate in every time slot, then $A_x = \lceil (1+x)^{1+\varepsilon} \rceil$, for all $x \geq 1$. One can check that, both these examples satisfy the conditions in assumption \textbf{A.2}}

\subsection{Regret Guarantee}
\label{sec:perf_synch_algo}

The regret guarantee of Algorithm \ref{algo:main-algo} is given in Theorem \ref{thm:strong_result}, which requires a definition.
Let $N \in \mathbb{N}$ and a $P$ be a $N \times N$ gossip matrix. Denote by the random variable $\tau_{spr}^{(P)}$ to be the \emph{spreading time} of a rumor in a pull model, with a rumor initially in node $1$ (cf \cite{shah_book}). Formally, consider a discrete time stochastic process where initially, node $1$ has a rumor. At each time step, each node $j \in \{1,\cdots,N\}$ that does not possess the rumor, calls another node sampled independently of everything else from the probability distribution $P(j,\cdot)$. If a node $j$ calls on a node possessing the rumor, node $j$ will possess the rumor at the end of the call (at the end of current time step). The spreading time $\tau_{spr}^{(P)}$ is the stopping time when all nodes possess the rumor for the first time. 


	
	\begin{theorem}
	\label{thm:strong_result}
	Suppose in a system of $N \geq 2$ agents  connected by a communication matrix $P$ satisfying assumption {\bf (A.1)} and $K \geq 2$ arms, each agent runs Algorithm \ref{algo:main-algo}, with UCB parameter $\alpha > 3$ and communication budget $(B_t)_{t \in \mathbb{N}}$ and $\varepsilon > 0$ satisfying assumption {\bf (A.2)}. Then the regret of any agent $i \in [N]$, after time any time $T \in \mathbb{N}$  is bounded by 
	\begin{equation}
	\mathbb{E}[R_T^{(i)}] \leq  \underbrace{\left( \sum_{j=2}^{\lceil \frac{K}{N} \rceil + 2}  \frac{1}{\Delta_j} \right) 4 \alpha \ln(T) + \frac{K}{4}}_{\text{{Collaborative UCB Regret}}} +  \underbrace{  g((A_x)_{x \in \mathbb{N}}) + \mathbb{E}[A_{2\tau_{spr}^{(P)}}]}_{\text{{Cost of Infrequent Pairwise Communications }}},
	\label{eqn:per-agent-regret-thm}
	\end{equation}
	where $(A_x)_{x \in \mathbb{N}}$ is given in Equation (\ref{eqn:A_x_defn}) and $g((A_x)_{x \in \mathbb{N}}) = A_{j^{*}}+ \frac{2}{2\alpha-3}\left(\sum_{l \geq \frac{j^{*}}{2}-1} \frac{A_{2l+1}}{A_{l-1}^{3}} \right)$ where
	\begin{multline*}
	   j^{*}=2 \max \bigg( A^{-1} \left( \left(N{K \choose 2} \left(\bigg \lceil \frac{K}{N} \bigg\rceil + 1\right) \right)^{\frac{1}{(2\alpha-6)}} \right) +1 , \min \left\{ j \in \mathbb{N} : \frac{A_j - A_{j-1}}{2 + \lceil \frac{K}{N}\rceil } \geq 1 + \frac{4 \alpha \log(A_j)}{\Delta_{2}^{2}} \right\}\bigg),
	\end{multline*} 
	where, $A^{-1}(x) = \sup \{ y \in \mathbb{N} : A_y \leq x\}$, $\forall x \in \mathbb{R}_{+}$.
\end{theorem}


{\color{black}

\subsection{Discussion}
In order to get some intuition from the Theorem, we consider a special case. Recall from Equation (\ref{eqn:A_x_defn}), that $A_x$ is the time slot when any agent pulls information for the $x$ th time. Thus, if for some $\beta > 1$, the communication budget $B_t = \lfloor t^{1/\beta} \rfloor$, then for all small $\varepsilon$ and all large $x$, the sequence $A_x = \lceil x^{\beta} \rceil$ . In other words, if communication budget scales polynomially (but sub-linearly) with time, then $A_x$ is also polynomial, but super linear. 
Similarly, if the gossip matrix corresponded to the complete graph, i.e., $P(i,j) = 1/N$, for all $i\neq j$ and $A_x = x^{\beta}$, we will show in the sequel (Corollary \ref{cor:regret-comm-main-paper}), that there exists an universal constant $C > 0$ such that $\mathbb{E}[A_{2 \tau_{spr}^{(P)}}] \leq  (C \log(N))^{\beta}$.
Thus, we have the following corollary.
\begin{corollary}
Suppose the communication budget satisfies $B_t = \lfloor t^{1/\beta} \rfloor$, for all $t \geq 1$, for some $\beta > 1$. Let $\varepsilon > 0$ be sufficiently small. Then the communication sequence $(A_x)_{x \in \mathbb{N}}$ in Equation (\ref{eqn:A_x_defn}) with $\varepsilon < \beta-1$ is such that $A_x = \lceil x^{\beta}\rceil,$ for all large $x$. If the gossip matrix connecting the agents corresponded to the complete graph, i.e., $P(i,j) = 1/N$, for all $i \neq j$, then under the conditions of Theorem \ref{thm:strong_result}, the regret of any agent $i \in \{1,\cdots,N\}$ at time $T \in \mathbb{N}$ satisfies 
\begin{multline*}
    \mathbb{E}[R_T^{(i)}] \leq \underbrace{\left(\sum_{j=2}^{\lceil \frac{K}{N} \rceil + 2}  \frac{1}{\Delta_j} \right) 4 \alpha \ln(T) + \frac{K}{4}}_{\text{{Collaborative UCB Regret}}} + \\  \underbrace{ \frac{4}{2\alpha-3} \frac{\pi^2}{6} 3^{\beta} + 4 \max \left( K^{\frac{3 }{ (2\alpha-6)}},\left(16\alpha \frac{2 + \lceil  \frac{K}{N}\rceil}{\Delta_2^2} \right)^{\frac{\beta}{\beta-1}} \right) + (C\log(N))^{\beta} }_{\text{{Cost of Infrequent Pairwise Communications }}},
\end{multline*}
\label{cor:main_result_interpretation}
where $C$ is an universal constant given in Corollary \ref{cor:regret-comm-main-paper}.
\end{corollary}
The proof is provided in Appendix \ref{appendix-proof-interpretation}. 
The terms denoting cost of pairwise communications correspond to the average amount of time any agent must wait before the best arm is in the playing set of that agent.
This cost can be decomposed into the sum of two dominant terms.
The term of order $\left( \frac{\lceil \frac{K}{N} \rceil}{\Delta_2} \right)^{\frac{2\beta}{\beta-1}}$
is the expected number of samples needed to identify the best arm by any agent. 
The term $(\log(N))^{\beta}$ is the amount of time taken by a pure gossip process to spread a message (the best arm in our case) to all agents, if the communication budget is given by $B_t = \lfloor t^{1/\beta}\rfloor$.

 
}
{\color{black}\subsection{Proof Sketch}}

The proof of this theorem is carried out in Appendix \ref{sec:synch_algo_proof} and we describe the main ideas here. We deduce in Proposition \ref{prop:weak_reg_decompose} that there exists a \emph{freezing} time $\tau$ such that, all agents have the best arm by time $\tau$ and only recommend the best arm from henceforth, i.e., the set of arms of agents do not change after $\tau$. The technical novelty of our proof is in bounding $\mathbb{E}[A_{\tau}]$, as this leads to the final regret bound (Proposition \ref{prop:weak_reg_decompose}). 
\\

There are two key challenges in bounding this term. First, the choice of arm recommendation is based on the most played arm in the current phase, while the choice of arm to pull is based on samples even in the past phases, as the UCB considers all samples of an arm thus far. If the phase lengths are large (Equation (\ref{eqn:A_x_defn}) ensures this), Lemma \ref{lem:error_estimate_ucb} shows that the probability of an agent recommending a sub-optimal arm at the end of a phase is small, irrespective of the number of times it was played till the beginning of the phase. Second, the events that any agent recommends a sub-optimal arm in different phases are not independent, as the reward samples collected by this agent, leading to those decisions are shared. We show in Proposition \ref{prop:strong_cost} by establishing that after a random, almost surely finite time (denoted as $\widehat{\tau}_{stab}$ in Appendix \ref{sec:synch_algo_proof}), agents never recommend incorrectly.

\subsection{Initialization without Agent IDs}
\label{sec:random_init}
The initialization in Line $2$ of Algorithm \ref{algo:main-algo} relies on each agent knowing its identity. However, in many settings, it may be desirable to have algorithms that do not depend on the agent's identity. We outline a simple procedure to fix this (with guarantees) in Appendix \ref{app:random_sticky}.

\section {Asynchronous \name \ Algorithm}
\label{sec:asynch_algo}
A synchronous system is not desirable in many cases as agents could get a large number of message requests during time slots $(A_j)_{j \geq 0}$. Consider an example where the gossip matrix $P$ is a star graph, i.e., for all $i \neq 1$, $P(i,1) = 1$ and $P(1,i) = \frac{1}{N-1}$. In this situation, at  time slots $(A_x)_{x \in \mathbb{N}}$, the central node $1$ will receive a (large) $N-1$ different requests for messages, which may be infeasible if agents are bandwidth constrained.
\\


We present an asynchronous algorithm to alleviate this problem. This new algorithm is identical to Algorithm \ref{algo:main-algo} with two main differences - {\em (i)} each agent chooses the number of time slots it stays in any phase $j$ as a random variable independently of everything else, and {\em (ii)} when asked for a recommendation, agents recommend the most played arm in the \emph{previous phase}. The first point, ensures that even in the case of the star graph described above, with high probability, eventually, no two agents will pull information in the same time slot. The second point ensures that even though the phase lengths are random, the \emph{quality} of recommendations are good as they are based on large number of samples. We give a pseudo-code in Algorithm \ref{algo:main-algo-prob} where lines $5$ and $18$ are new and lines $8$ (agents have different phase lengths) and $9$ (arm recommendation from previous phase) are modified from Algorithm \ref{algo:main-algo}.


\begin{algorithm*}[t]
	\caption{Asynch \name \; Algorithm (at Agent $i$)}
	\begin{algorithmic}[1]
		\State \textbf{Input}: Communication Budget $(B_t)_{t \in \mathbb{N}}, \text{ UCB Parameter }\alpha, \text{ Slack }\delta$, $\varepsilon > 0$
		\State \textbf{Initialization}:  $ \widehat{S}^{(i)},S_0^{(i)}$ according to Equations (\ref{eqn:hat_S_i}) and (\ref{eqn:S_i_0}) respectively.
		\State $j\gets 0$ 
        \State  $A_j = \max \left(\min \{t \geq 0, B_t \geq j\}, \lceil (1+j)^{1+\varepsilon}\rceil \right)$ \Comment{{Reparametrize the communication budget}}
        \State $\mathcal{P}_j^{(i)} \sim \text{Unif}[(A_j-A_{j-1}), (1+\delta)(A_j-A_{j-1})]$ \Comment{{Uniformly distributed phase length}}
		\For{Time $t \in \mathbb{N}$}
		\State Pull - $\arg\max_{l \in S_i^{(j)}} \left( 
		\widehat{\mu}_l^{(i)}(t-1) + \sqrt{\frac{\alpha \ln(t)}{T_l^{(i)}(t-1)}}\right)$
		
		\If {$t == \sum_{y=0}^{j}\mathcal{P}_{y}^{(i)}$}
		\State $\mathcal{O}_j^{(i)} \gets$ {\ttfamily GET-ARM-PREV}($i,t$)
		\If {$\mathcal{O}_j^{(i)} \not \in S_i^{(j)}$}
		\State $U_{j+1}^{(i)} \gets \arg\max_{l \in \{U_j^{(i)},L_j^{(i)}\}}\left( T_l\left(\sum_{y=0}^{j}\mathcal{P}_{y}^{(i)}\right)- T_l\left(\sum_{y=0}^{j-1}\mathcal{P}_{y}^{(i)} \right)\right)$ \Comment{{ Most played arm in current phase}}
		\State $L_{j+1}^{(i)} \gets \mathcal{O}_j^{(i)}$ 
		\State $S_{j+1}^{(i)} \gets \widehat{S}^{(i)} \cup L_{j+1}^{(i)} \cup U_{j+1}^{(i)}$ \Comment{{Update set of playing arms}}
		\Else
		\State $S_{j+1}^{(i)} \gets S_j^{(i)}$.
		\EndIf 
		\State $j \gets j+1$ 
\State  $A_j = \max \left(\min \{t \geq 0, B_t \geq j\}, \lceil (1+j)^{1+\varepsilon}\rceil \right)$ \Comment{{Reparametrize the communication budget}}
        \State $\mathcal{P}_j^{(i)} \sim \text{Unif}[(A_j-A_{j-1}), (1+\delta)(A_j-A_{j-1})]$ \Comment{Update next phase length}
		\EndIf
		\EndFor
	\end{algorithmic}
	\label{algo:main-algo-prob}
\end{algorithm*}

\begin{algorithm}[t]
	\caption{Asynch Arm Recommendation}
	\begin{algorithmic}
		\Procedure {Get-Arm-Prev}{($i$,$t$)} \Comment{\emph{Input  an agent $i$ and time $t$}}
	
		\State $m \sim P(i,\cdot)$ \Comment{\emph{Sample another agent}}
		\State $j \gets \inf \{r \geq 0: \sum_{y=0}^{r}\mathcal{P}_{y}^{(m)} \geq t \}$  \Comment{ \emph{Phase of agent $m$ at time $t$}}
		\State $\mathcal{Y}_{j-1}^{(m)} \gets \sum_{y=0}^{j-1}\mathcal{P}_{y}^{(m)}$, $\mathcal{Y}_{j-2}^{(m)} \gets \sum_{y=0}^{j-2}\mathcal{P}_{y}^{(m)}$ \\
		\Return $\arg\max_{l \in S_m^{(j-1)}} \bigg( T_l^{(m)}(\mathcal{Y}_{j-1}^{(m)}) - T_l^{(m)}(\mathcal{Y}_{j-2}^{(m)})\bigg)$ \Comment{\emph{Most played arm in phase $j-1$ }}
		\EndProcedure
	\end{algorithmic}
	\label{algo:get-prev-arm}
\end{algorithm}


\begin{theorem}

	Suppose in a system of $N \geq 2$ agents  connected by a communication matrix $P$ satisfying assumption {\bf (A.1)} and $K \geq 2$ arms, each agent runs Algorithm \ref{algo:main-algo-prob}, with UCB parameter $\alpha > 3$, $\delta >0$ and communication budget $(B_t)_{t \in \mathbb{N}}$ and $\varepsilon > 0$ satisfying assumption {\bf (A.2)}. Then the regret of any agent $i \in [N]$, after any time $T \in \mathbb{N}$ is bounded by 
	\begin{equation*}
	\mathbb{E}[R_T^{(i)}] \leq \underbrace{\left( \sum_{j=2}^{\lceil \frac{K}{N} \rceil + 2}  \frac{1}{\Delta_j} \right) 4 \alpha \ln(T) + \frac{K}{4}}_{\text{{Collaborative UCB Regret}}} +  \underbrace{ (1+\delta)\mathbb{E}[A_{2\lfloor 2+\delta\rfloor  \tau_{spr}^{(P)}}] + \widehat{g}((A_x)_{x \in \mathbb{N}},\delta)}_{\text{{Cost of Asynchronous Infrequent Pairwise Communications}}}, 
	\end{equation*}
	where $\widehat{g}((A_x)_{x \in \mathbb{N}},\delta) =  	2(1+\delta)\bigg(A_{2\lceil 2+\delta \rceil j^{*}} + \left(\frac{2}{2\alpha-3}\right)\sum_{l \geq3} \frac{A_{2l}}{A_{l-1}^{3}}  \bigg)$, where $j^{*}$ given in Theorem \ref{thm:strong_result} and $(A_x)_{x \in \mathbb{N}}$ is given in Equation (\ref{eqn:A_x_defn}).
	\label{thm:prob_main_result}
\end{theorem}

{\subsection{Proof Sketch}}
The proof of this theorem is carried out in Appendices \ref{sec:poisson_asynch_algo},\ref{sec:proof_poisson_result} and
\ref{sec:prob_algo_proof}. In order to prove this, we find it effective to give a more general algorithm (Algorithm \ref{algo:poisson-algo} in Appendix \ref{sec:poisson_asynch_algo}) where the agents choose the phase lengths $\mathcal{P}_j$ as a Poisson distributed random variable. This algorithm does not satisfy the budget constraint exactly, but only in expectation, over the randomization used in the algorithm. We analyze this in Theorem \ref{thm:poisson_result} stated in Appendix \ref{sec:poisson_asynch_algo} and proved in Appendix \ref{sec:proof_poisson_result}. The main additional technical challenge is that the phase lengths of different agents are staggered. We crucially use the convexity of the sequence  $(A_x)_{x \in \mathbb{N}}$ (Assumption {\bf A.2}) in Proposition \ref{prop:convex}, along with more involved coupling argument to a rumor spreading process (Proposition \ref{prop:couple_spread}).
The proof of Theorem \ref{thm:prob_main_result} is a corollary of the proof of Theorem \ref{thm:poisson_result} in Appendix \ref{sec:prob_algo_proof}.
\\


\section{Lower Bound}
\label{sec:lower_bound}
In order to state the lower bound, we will restrict ourselves to a class of \emph{consistent policies} \cite{lai_robbins}. A policy (or algorithm) is consistent, if {\color{black}for} any agent $i \in [N]$, and any sub-optimal arm $l \in \{2,\cdots,K\}$, the expected number of times agent $i$ plays arm $l$ up-to time $t \in \mathbb{N}$ (denoted by $T_l^{(i)}(t)$) satisfies for all $a > 0$, $\lim_{t \rightarrow \infty} \frac{\mathbb{E}[T_l^{(i)}(t)]}{t^a} = 0$. 
\begin{theorem}
The regret of any agent $i \in [N]$ after playing arms for $T$ times under any {\color{black} consistent policy} played by the agents and any communication matrix $P$ satisfies
\begin{equation}
    \liminf_{T \rightarrow \infty} \frac{\mathbb{E}[R_{T}^{(i)}]}{\ln(NT)} \geq  \left( \frac{1}{N}\sum_{j =2}^{K} \frac{\Delta_j}{\text{KL}(\mu_j,\mu_1)} \right),
    \label{eqn:full_interaction_lb_m}
    \end{equation}
    
where for any $a,b \in [0,1]$, $\text{KL}(a,b)$ is the Kullback-Leibler distance between two Bernoulli distributions with mean $a$ and $b$.

\label{thm:full_interaction}
\end{theorem}

The proof of the theorem is carried out in Appendix \ref{sec:lb_proof}. The proof of this lower bound is based on a system where there are no communication constraints. From standard inequalities for KL divergence, we get from Equation (\ref{eqn:full_interaction_lb_m}) that
\begin{equation}
    \liminf_{T \rightarrow \infty} \frac{\mathbb{E}[R_{T}^{(i)}]}{\ln(NT)} \geq  \mu_1(1-\mu_1)\left( \frac{1}{N}\sum_{j =2}^{K} \frac{1}{\Delta_j} \right).
    \label{eqn:full_interaction_lb}
    \end{equation}

\section{Insights}
\label{sec:tradeoffs}




{\bf1. Insensitivity to Communication Constraints - }
 The following corollary follows directly from Theorems \ref{thm:strong_result} and \ref{thm:prob_main_result}.

 \begin{corollary}
 \label{cor:asym_regret}
  	Suppose in a system of $N \geq 2$ agents each running Algorithm \ref{algo:main-algo} or \ref{algo:main-algo-prob} with parameters satisfying conditions in Theorems \ref{thm:strong_result} and \ref{thm:prob_main_result} respectively. Then, for every agent $i \in [N]$ and time $T \in \mathbb{N}$,
 	\begin{align*}
 	   \limsup_{T \rightarrow \infty} \frac{E[R_T^{(i)}]}{\ln(T)} \leq \left( \sum_{j=2}^{\lceil \frac{K}{N} \rceil + 2} \frac{4\alpha}{\Delta_j} \right).
 	\end{align*}
 	\label{cor-com-insensitive}
 \end{corollary}
Thus, as long as the gossip matrix $P$ is connected (Assumption {\bf A.1}) and the communication budget over a horizon of $T$ is at-least $\Omega (\log(T))$, (Assumption {\bf A.2}), the asymptotic regret of any agent, is insensitive to $P$ and the communication budget. 
\\

{\bf \noindent 2. Benefit of Collaboration - } As an example, consider a system where $K=N$ and arm-means such that $\forall j \in \{2,\cdots,K\}$, $\Delta_j := \Delta > 0$. Let $\boldsymbol{\Pi}$ be any consistent policy for the agents in the sense of Theorem \ref{thm:full_interaction}. Then Equation (\ref{eqn:full_interaction_lb}) and Corollary \ref{cor:asym_regret} implies that $\sup_{\boldsymbol{\pi}}\limsup_{T \rightarrow \infty}  \frac{\mathbb{E}[R_T^{(i)}]}{\mathbb{E}_{\boldsymbol{\pi}}[R_T^{(i)}]} \leq \frac{16 \alpha}{\mu_1(1-\mu_1)}$, where in the numerator is the regret obtained by our algorithms and the denominator is by the policy $\boldsymbol{\pi}$. As ratio of asymptotic regret in our algorithm and the lower bound is a constant \emph{independent of the size of the system}, (does not grow with $N$), our algorithms benefit from collaboration. Recall that the lower bound is obtained from the full interaction setting where all agents communicate with every other agent, after every arm pull while in our model, every agent pulls information, a total of at most $o(T)$ times over a time horizon of $T$. Thus, we observe that, despite communication constraints, any agent in our algorithm performs nearly as good as the best possible algorithm when agents have no communication constraints, i.e., the regret ratio is a constant independent of $N$.
\\

{\bf \noindent 3. Impact of Gossip Matrix $P$ - }  The second order constant term in the regret bounds in Theorems \ref{thm:strong_result} and \ref{thm:prob_main_result} provides a way of quantifying the impact of $P$, based on its conductance, which we define now. Given an undirected finite graph $G$ on vertex set $V$, denote for any vertex $u \in V$, $\text{deg}(u)$ to be the degree of vertex $u$ in $G$. For any set $H \subseteq V$, denote by $\text{Vol}(H) = \sum_{u \in H} \text{deg}(u)$. For any two sets $H_1,H_2 \subseteq V$, denote by $\text{Cut}(H_1,H_2)$, to be the number of edges in $G$ with one end in $H_1$ and the other in $H_2$. The conductance of $G$, denoted by $\phi$ is defined as 
\begin{align*}
    \phi := \min_{H \subset V: 0 < \text{Vol}(H) \leq \text{Vol}(V)/2} \frac{\text{Cut}(H,V\setminus H)}{\text{Vol}(H)}.
\end{align*}
The following corollary, illustrates the intuitive fact that if the conductance of the gossip matrix is higher, then the regret (the second order constant term) is lower. {\color{black} For sake of clarity, we give the corollary in the special case of polynomially scaling communication budgets and provide a general result in the Appendix in Corollary \ref{cor:regret-comm}.}

\begin{corollary}
Suppose $N \geq 2$ agents are connected by a $d$-regular graph with adjacency matrix $\boldsymbol{A}_G$ having conductance $\phi$ and the gossip matrix $P = d^{-1}\boldsymbol{A}_G$. Suppose the communication budget scales as $B_t = \lfloor t^{1/\beta} \rfloor$, for all $t \geq 1$, where $\beta > 1$ is arbitrary. 
If the agents are using Algorithm \ref{algo:main-algo}
with parameters satisfying assumptions in Theorem \ref{thm:strong_result}, then for any $i \in [N]$ and $T \in \mathbb{N}$ 
\begin{align*}
\mathbb{E}[R_T^{(i)}] \leq \underbrace{{4 \alpha \ln(T)}\left( \sum_{j=2}^{\lceil \frac{K}{N} \rceil + 2} \frac{1}{\Delta_j} \right) + \frac{K}{4}}_{\text{\clap{Collaborative UCB Regret}}} +    \underbrace{\left(2C\frac{ \log(N)}{\phi}\right)^{\beta}}_{\text{\clap{Impact of Gossip Matrix}}} +  \underbrace{  2\frac{3^{\beta}}{2\alpha-3} \frac{\pi^2}{6} + (j^*)^{\beta}+1}_{\text{\clap{Constant Independent of $P$}}},
\end{align*}
where $j^*$ is a constant independent of the gossip matrix $P$, depending only on $N,K$ and $\Delta_2$ (given in Theorem \ref{thm:strong_result}).

\label{cor:regret-comm-main-paper}
\end{corollary}



\noindent The proof is provided in Appendix \ref{appendix_proof_cor_reg_comm_tradeoff}. {\color{black} Notice, that the only term in the regret that depends on the graph is the conductance $\phi$. 
In order to derive some intuition,
we consider two examples - one wherein the $N$ agents are connected by a complete graph, and one wherein they are connected by the ring graph. The conductance of the complete graph is $\frac{N}{2(N-1)}$, while that of the ring graph is $\frac{2}{N}$. Thus, the cost of communications scales as $(4C \log(N))^{\beta}$
for the complete graph, but scales as $(4C \log(N) N )^{\beta}$ in the ring graph. This shows the reduction in regret that is possible by a `more' connected gossip matrix, where the regret is reduced from order $(N\log(N))^{\beta}$ to $(\log(N))^{\beta}$ on moving from the ring graph to the complete graph. This is also demonstrated empirically in Figures \ref{fig1} and \ref{fig2}.}
\\

{\bf \noindent 4. Regret/Communication Trade-off -} For a fixed problem instance and gossip matrix $P$, reducing the total number of information pulls, i.e., reducing the rate of growth of $(B_x)_{x \in \mathbb{N}}$ increases the per-agent regret. This can be inferred by examining the cost of communications in Equation (\ref{eqn:per-agent-regret-thm}), which we state in the following corollary.
\begin{corollary}
\label{cor:reg-comm-tradeoff}
Suppose, Algorithm \ref{algo:main-algo} is run with $K$ arms and $N$ agents connected by a gossip matrix $P$, with two different communication schedules $(A_x^{(1)})_{x \in \mathbb{N}}$ and $(A_x^{(2)})_{x \in \mathbb{N}}$, such that $\lim_{x \to \infty}\frac{A_x^{(1)}}{A_x^{(2)}} = 0$. Then there exist positive constants $N_0,K_0 \in \mathbb{N}$ (depending on the two communication sequences), such that for all $N \geq N_0$ and $K \geq K_0$, and $P$, the cost of communications in the regret bound in Equation (\ref{eqn:per-agent-regret-thm}) is ordered as 
\begin{align*}
    g((A_x^{(1)}))+\mathbb{E}[A_{2 \tau_{spr}^{(P)}}^{(1)}] \geq     g((A_x^{(2)}))+\mathbb{E}[A_{2 \tau_{spr}^{(P)}}^{(2)}].
\end{align*}
\end{corollary}
The proof of this corollary is provided in the Appendix \ref{appendix-com-reg-tradeoff}. 
In light of Equation (\ref{eqn:per-agent-regret-thm}) in Theorem \ref{thm:strong_result}, the above corollary makes precise the qualitative fact that if agents are allowed more communication budget, then they experience lesser regret. We demonstrate this empirically in Figure \ref{fig3}.


\section{Numerical Results}
\label{sec:simulations}

We evaluate our algorithm and the insights empirically.
Each plot is the regret averaged over all agents, produced after $30$ and $100$ random runs for Algorithms \ref{algo:main-algo} and  Algorithm \ref{algo:main-algo-prob} (with $\delta =0.5$) respectively, along with $95\%$ confidence intervals.
We also plot the two benchmarks of no interaction among agents (where a single agent is running the UCB-$4$ algorithm of \cite{ucb_auer}) and the system corresponding to complete interaction, where all agents are playing the UCB-$4$ algorithm with entire system history of all arms pulled and rewards obtained by all agents as described in Section \ref{sec:lower_bound}. 
\\

\noindent {\bf Synthetic Experiments} -
We consider a synthetic setup with $\Delta=0.1$, $\mu_1 = 0.95, \mu_2 =0.85$, rest of the arm means sampled uniformly in $(0, 0.85]$. In Figures \ref{fig1} and \ref{fig2}, we consider the impact of gossip matrix by fixing the communication budget $B_t = \lfloor t^{1/3} \rfloor$ ($A_x = x^3$) and varying $P$ to be the complete and cycle graph among agents. We see that our algorithms are effective in leveraging collaboration in both settings and experiences a lower regret in the complete graph case as opposed to the cycle graph, as predicted by our insights.
\\

\begin{figure}[t]
\centering
\includegraphics[width=\columnwidth]{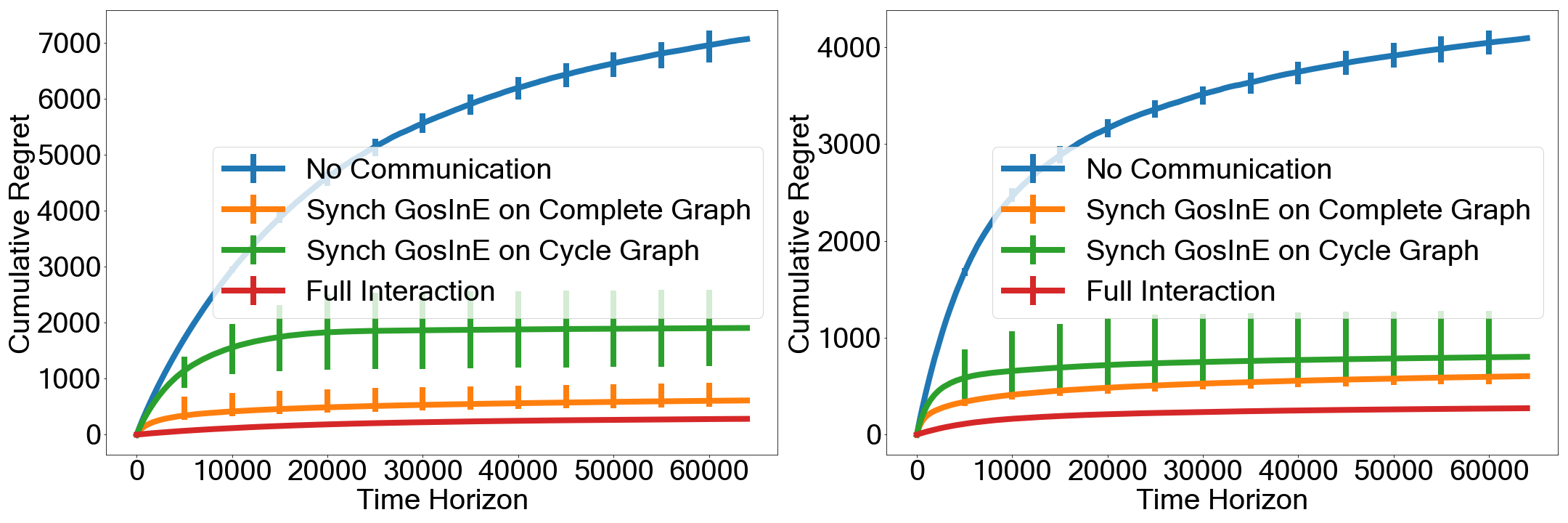}
\caption{$(N, K)$ are $(25, 75)$ and $(15, 50)$ respectively.}
\label{fig1}
\end{figure}

\begin{figure}[t]
\centering
\includegraphics[width=\columnwidth]{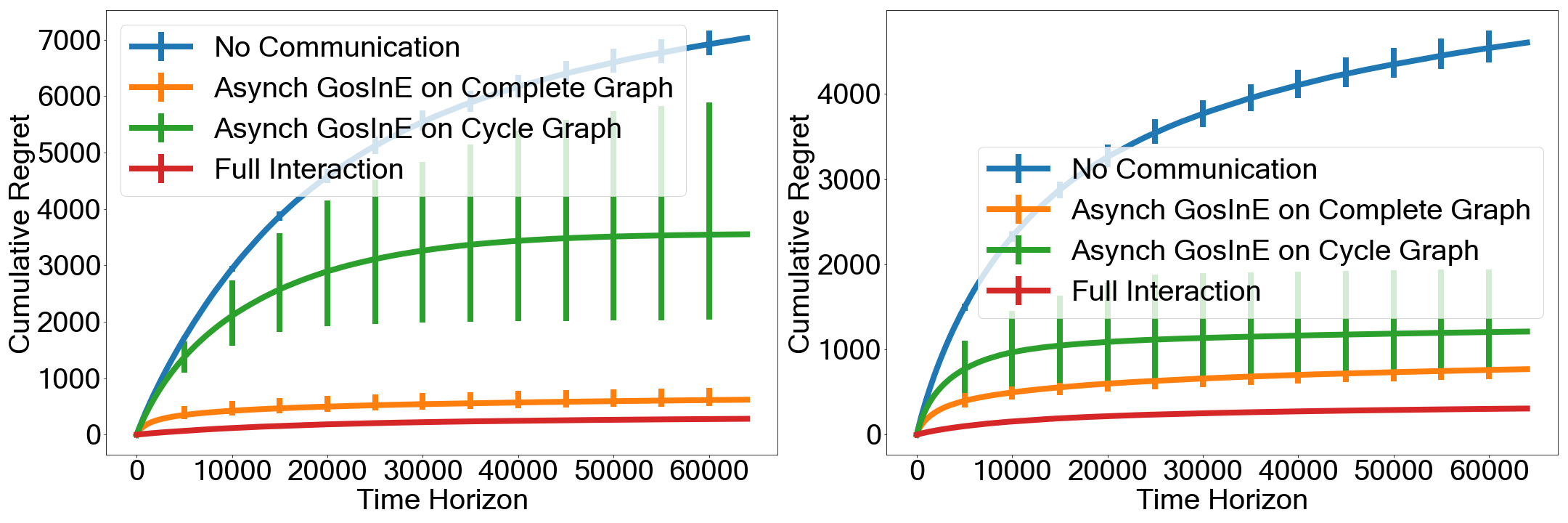}
\caption{ $(N, K)$ are $(25, 75)$ and $(15, 50)$ respectively.}
\label{fig2}
\end{figure}

In Figure \ref{fig3}, we compare the effect of communication budget by considering two scenarios - polynomial budget $B_t = \lfloor t^{1/3} \rfloor$ ($A_x = x^3$) and logarithmic budget $B_t = \lfloor \log_2(t) \rfloor$ ($A_x = 2^x$). We see that even under a logarithmic communication budget, our algorithms achieve significant regret reduction.
\\

\noindent {\bf Real Data} - In Figure \ref{fig:sigmetrics-comparison}, we run our Algorithms on MovieLens data \cite{movielens} using the methodology in \cite{sigmetrics}. This dataset contains $6k$ movies rated by $4k$ users. We treat the movies as arms and estimate the arm-means from the data by averaging the ratings of a section of similar users (same age, gender and occupation and have rated at-least $30$ movies). We further select only those movies that have at least $30$ ratings by users in the chosen user category. We estimate the missing entries in the sub-matrix (of selected users and movies) using matrix completion \cite{fancyimpute} and choose a random set of $30$ and $40$ movies, in Figure \ref{fig:sigmetrics-comparison}. We compare against \cite{sigmetrics} (hyperparameter $\varepsilon = 0.1$) for the setting of complete graph among agents and communication budget $B_t = \lfloor t^{1/3} \rfloor$. We see that in all settings, our algorithm has superior performance and strongly benefits from limited collaboration.



\begin{figure}[t]
\centering
\includegraphics[width=\columnwidth]{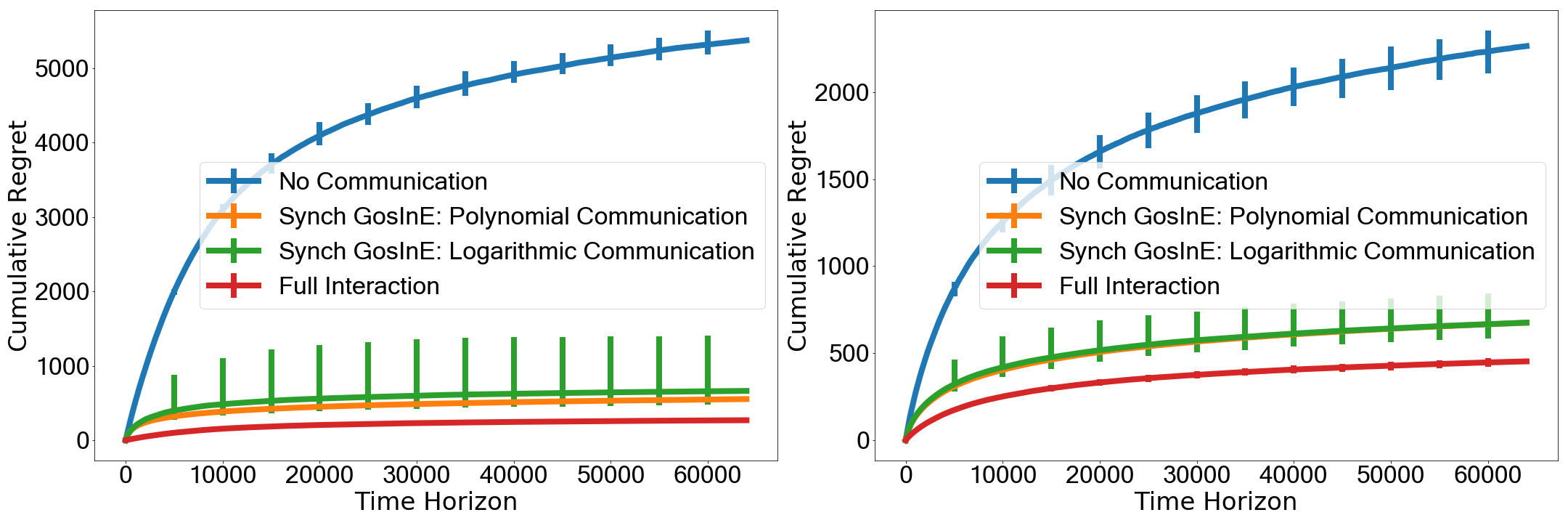}
\caption{ $(N, K)$ as $(20, 70)$  and $(5, 20)$ and the graphs are complete and cycle respectively.}
\label{fig3}
\end{figure}



\section{Related Work}
\label{sec:related_work}

The closest to our work is \cite{sigmetrics} which introduced a model similar to ours. However, the present paper improves on the algorithm in \cite{sigmetrics} in three aspects: {\em (i)} our algorithm can handle any gossip matrix $P$, while that of \cite{sigmetrics} can only handle complete graphs and {\em (ii)}, the algorithm in \cite{sigmetrics}, needs as an input, a lower bound on the arm gap between the best and the second best arm, while our algorithms do not require any such knowledge and {\em (iii)},
our regret scaling is superior even on complete graphs. 



 
\begin{figure}[t]
\centering
\includegraphics[width=\columnwidth]{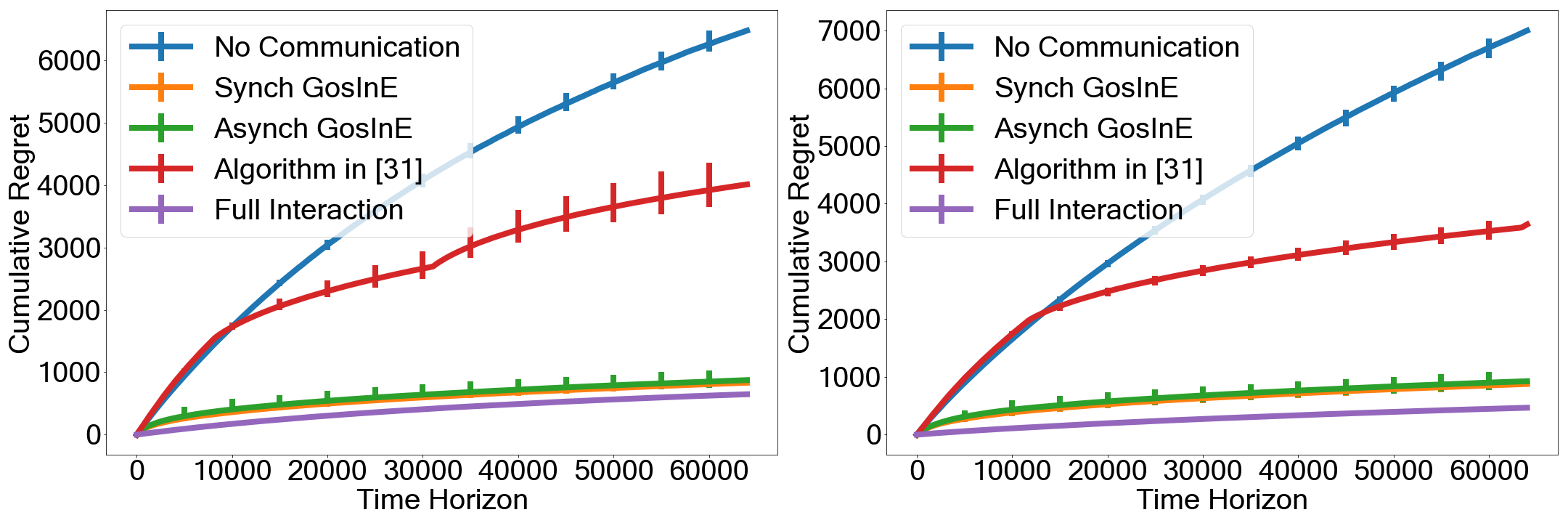}
\caption{ $(N, K)$ are $(10, 30)$ and $(15, 40)$ respectively.}
\label{fig:sigmetrics-comparison}
\end{figure} 
 
The multi-agent MAB was first introduced in the non-stochastic setting in \cite{kleinberg_collaborative_learning} and further developed in  \cite{delay_nonstochastic}. However, there was no notion of communication budgets in these models. Subsequently, \cite{kanade2012distributed} considered the regret/communication trade-off in the non-stochastic setting, different from our stochastic MAB model. In the stochastic setting, the papers of \cite{stochastic_team},\cite{buccapatnam}, \cite{kanade_stochastic}, \cite{Kolla2018}, \cite{frequentist_bayesian} consider a collaborative multi agent model where agents minimize individual regret in a decentralized manner. In these models, communications is not an active decision made by agents, rather agents can observe neighbor's actions and are, therefore, different from our setup, where agents actively choose to communicate depending on a budget. 
The papers of \cite{hillel} and \cite{p2p_simple_regret} study the benefit of collaboration in reducing simple regret, unlike the cumulative regret considered in our paper. The paper of \cite{dist_clustering} considers a distributed version of contextual bandits, in which agents could share information, whose length grows with time and thus different from our setup.
There has also been a lot of recent interest in `competitive' multi-agent bandits (\cite{anima_distributed}, \cite{distributed_no_com}, \cite{musical_chair}, \cite{avner_cognitive},
\cite{decentralized_multi_agent_collision},
\cite{distributed_no_communication},\cite{mansour_compete},\cite{matching_market}), where if multiple agents choose the same arm in a time slot, then they experience a `collision' and receive small reward (only a subset (possibly empty) gets a reward). This differs from our setup where even on collision, agents receive independent rewards.



\section{Conclusions}


We introduced novel algorithms for multi agent MAB, where agents play from a subset of arms and recommend arm-IDs. Our algorithms leverage collaboration effectively and in particular, its performance (asymptotic regret) is insensitive to the communication constraints. Furthermore, our algorithm exhibits a regret communication trade-off, namely achieves lower regret (finite time) with increased communications (budget or conductance of $P$), which we characterize through explicit bounds.


\subsection*{Acknowledgements}

This work was partially supported by ONR Grant N00014-19-1-2566, NSF Grant SATC 1704778, ARO grant W911NF-17-1-0359 and the NSA SoS Lablet H98230-18-D-0007. AS also thanks Fran\c cois Baccelli for the support and generous funding through the Simons Foundation Grant (\#197892) awarded to the University of Texas at Austin.

\bibliographystyle{plain}
\bibliography{ref_mab_social}

\begin{appendices}
\onecolumn









\section{Discussion on Technical Assumptions in Section \ref{sec:assumptions}}
\label{sec:assumption_discussion}
 Assumption \textbf{A.1} states that the graph of communication among agents is connected. Observe that if \textbf{A.1} is not satisfied, then there exists at-least a pair of agents that can never exchange information among each other, making the setup degenerate. Assumption \textbf{A.2} implies that, any agent over a time interval of $T$ arm-pulls, can engage in information-pulls,  at-least $\Omega(\log(T))$ times. The convergence of the series in \textbf{A.2} also hold true for all `natural' examples, such as exponential and polynomial. For instance, the series is convergent if for all large $l$, either $B_l = \lceil \frac{1}{D}\log^{\beta}(l) \rceil $ or $B_l =\lceil  l^{1/(D+1)} \rceil$, for all $D > 0$ and $\beta > 1$. Thus, conditions \textbf{A.1} and \textbf{A.2} do not impact any practical insights we can draw from our results.

\section{Proof of Theorem \ref{thm:strong_result}}
\label{sec:synch_algo_proof}

In order to give the proof, we first set some notations and definitions.  We make explicit a probability space construction from \citep{lattimore_book}, that makes the proof simpler. We assume that there is a sequence of independent $\{0,1\}$ valued random variables $({Y}_j^{(i)}(t))_{i \in [N], j \in [K], t \geq 0}$, where for every $j \in [K]$, the collection $(Y_j^{(i)}(t))_{t \geq 0, i \in [N]}$ is an i.i.d. Bernoulli random variable of mean $\mu_j$. The interpretation being that if an agent $i$ pulls arm $j$ for the $l$th time, it will receive reward $Y_j^{(i)}(l)$. Additionally, we also have on the probability space a sequence of independent $[N]$ valued random variables $(Z_j^{(i)})_{j \geq 0, i\in [N]}$, where for each $i \in [N]$, the sequence $(Z_j^{(i)})_{j \geq 0}$ is iid distributed as $P(i,\cdot)$. The interpretation is that when agent $i$ wishes to receive a recommendation at the end of phase $j$, it will do so from agent $Z_j^{(i)}$. 


\subsection{Definitions and Notations}

In order to analyze the algorithm, we set some definitions. Let $\mathcal{B}_{j}^{(i)}$ to be the best arm in $S_{j}^{(i)}$, i.e., $\mu_{\mathcal{B}_{j}^{(i)}} = \max_{l \in S_j^{(i)}}\mu_l$. Observe that since the set $S_j^{(i)}$ is random,  $\mathcal{B}_{j}^{(i)}$ is also a random variable. For every agent $i \in [N]$ and phase $j \geq 0$, we denote by $\widehat{\mathcal{O}}_j^{(i)} \in S_j^{(i)}$ to be that arm, that agent $i$ played the most in phase $j$. Note, from the algorithm, if any agent $i^{'}$ pulled an arm from agent $i$ at the end of phase $j$ for a recommendation, it would have received arm $\widehat{\mathcal{O}}_j^{(i)} $.
\\


 Fix an agent $i \in [N]$ and phase $j \geq 0$. Let $\mathcal{S}^{(i)}$ be a collection of all subsets $S \subset [K]$ of cardinality $|S| = \lceil \frac{K}{N} \rceil + 2$, such that $1 \in S, \widehat{S}^{(i)} \subset S$. For any $S \in \mathcal{S}^{(i)}$, index the elements in $S$ as $\{l_1,\cdots,l_{\lceil \frac{K}{N} \rceil +2} \}$ in increasing order of arm-ids. Let $a_1,\cdots a_{\lceil \frac{K}{N} \rceil +2} \in \mathbb{N}^{\lceil \frac{K}{N} \rceil +2}$ be such that $\sum_{m =0}^{\lceil \frac{K}{N} \rceil +2} a_m \geq 0$. For every agent $i\in[N]$, phase $j \geq 0$ and $(a_1,\cdots,a_{\lceil \frac{K}{N}\rceil +2}) \in \mathbb{N}^{\lceil \frac{K}{N} \rceil +2}$, denote by the event $\xi_j^{(i)}(S; a_1,\cdots,a_{\lceil \frac{K}{N} \rceil})$  as 
 \begin{align*}
 \xi_j^{(i)}(S; a_1,\cdots,a_{\lceil \frac{K}{N} \rceil}) := \left\{ S_j^{(i)} = S, T_{l_1}(A_{j-1}) = a_1, \cdots, T_{l_{\lceil \frac{K}{N}\rceil+2}(A_{j-1})} = a_{\lceil \frac{K}{N}\rceil+2}, \widehat{\mathcal{O}}_j^{(i)} \neq 1 \right\}.
 \end{align*}
Denote by $\Xi_j^{(i)}$ as the union of all such events, i.e.,
\begin{align*}
\Xi_j^{(i)} := \bigcup_{S \in \mathcal{S}^{(i)}} \left( \bigcup_{\left(a_1,\cdots a_{\lceil \frac{K}{N} \rceil +2}\right) \in \mathbb{N}^{\lceil \frac{K}{N} \rceil +2}}  \xi_j^{(i)}(S; a_1,\cdots,a_{\lceil \frac{K}{N} \rceil}) \right),
\end{align*}
and by $\chi_j^{(i)}$ its indicator random variable, i.e.,
\begin{align}
\chi_j^{(i)} = \mathbf{1}_{\Xi_j^{(i)}}.
\label{eqn:chi_defn_determinisstic}
\end{align}

In words, the event $\chi_j^{(i)}$ is the indicator variable indicating whether agent $i$ does not recommend the best arm at the end of phase $j$, under \emph{some sample path}, i.e., we take an union over all possible set of playing arms that contain arm $1$ (i.e., set $\mathcal{S}^{(i)}$) and all possible number of plays of the various arms in $S$ until the beginning of phase $j$ (i.e., the set of histories in $\mathcal{A}_j$). In Lemma \ref{lem:error_estimate_ucb}, we provide an upper bound to this quantity. Notice from the construction that for each agent $i \in [N]$ and phase $j \geq 0$, the random variable $\chi_j^{(i)}$ is measurable with respect to the reward sequence $({Y}_j^{(i)}(t))_{ j \in [K], t \in [0,A_j]}$. Also, trivially by definition, observe that $\chi_j^{(i)} \geq \mathbf{1}_{\widehat{\mathcal{O}}_j^{(i)} \neq 1, 1 \in S_j^{(i)}}$ almost-surely. This is so since $\chi_j^{(i)}$ is an union bound over all possible realizations of the communication sequence and reward sequence of other agents, while $\mathbf{1}_{\widehat{\mathcal{O}}_j^{(i)} \neq 1, 1 \in S_j^{(i)}}$ considers a particular realization of the communication and rewards of other agents.
\\

We now define certain random times that will be useful in the analysis.
 
\begin{align*}
\widehat{\tau}_{stab}^{(i)} &= \inf \{ j^{'} \geq j^{*} : \forall j \geq j^{'}, \chi_j^{(i)} = 0 \}, \\
\widehat{\tau}_{stab} &= \max_{i \in [N]} \widehat{\tau}_{stab}^{(i)},\\
\widehat{\tau}_{spr}^{(i)} &= \inf\{j \geq \widehat{\tau}_{stab} : 1 \in S_j^{(i)} \} - \widehat{\tau}_{stab}, \\
\widehat{\tau}_{spr} &= \max_{i \in \{1,\cdots,N\}} \widehat{\tau}_{spr}^{(i)}, \\
{\tau} &= \widehat{\tau}_{stab}+\widehat{\tau}_{spr}.
\end{align*}

In words, $\widehat{\tau}_{stab}^{(i)}$ is the earliest phase such that, for all subsequent phases, if agent $i$ has the best arm, then it will recommend the best arm. The time $\widehat{\tau}_{spr}^{(i)}$ is the number of phases it takes after $\widehat{\tau}_{stab}$ for agent $i$ to have arm $1$ in its playing set. The following proposition follows from the definition of the random times.

\begin{prop}
	 For all agents $i \in \{1,\cdots,N\}$, we have almost-surely,
	\begin{align*}
	\bigcap_{j \geq{\tau}} S_j^{(i)} &= S_{{\tau}}^{(i)}, \\
	\widehat{\mathcal{O}}_l^{(i)} &= 1 \text{ }\forall l \geq {\tau}, \text{ } \forall i \in \{1,\cdots,N\}.
	\end{align*}
	\label{prop:strong_freeze}
\end{prop}
\begin{proof}
	
	Fix any agent $i \in [N]$ and any phase $j \geq \tau$. Since $\tau \geq \widehat{\tau}_{stab}^{(i)}$, we have for all $j \geq \tau$, 
	\begin{equation}
	\chi_j^{(i)} = 0.
	\label{eqn:freeze_proof1}
	\end{equation}
	Furthermore, from the definition of $\chi_j^{(i)}$, we know that 
	\begin{equation}
	\chi_{j}^{(i)}  \geq \mathbf{1}_{1 \in S_j^{(i)},\widehat{\mathcal{O}}_j^{(i)} \neq 1},
	\label{eqn:freeze_proof2}
	\end{equation}
	almost-surely.  However, as $\tau \geq \widehat{\tau}_{spr}^{(i)} + \widehat{\tau}_{stab}$, we know that
	\begin{equation}
	1 \in S_j^{(i)}.
	\label{eqn:freeze_proof3}
	\end{equation}
	 Thus, from Equations (\ref{eqn:freeze_proof1}), (\ref{eqn:freeze_proof2}) and (\ref{eqn:freeze_proof3}), we have that $\widehat{\mathcal{O}}_j^{(i)} = 1$. Since $j \geq \tau$ was arbitrary, we have that for all $j \geq \tau$, $\widehat{\mathcal{O}}_j^{(i)} = 1$. Since agent $i \in [N]$ was arbitrary, we have that for all agents $i \in [N]$ and all phases $j \geq \tau$, we have $\widehat{\mathcal{O}}_j^{(i)}$=1. From the Algorithm, we know that any agent will change its set of arms only if the recommendation it receives is not present in the playing set (see line $8$ of Algorithm \ref{algo:main-algo}). The preceding argument says that is not the case and hence for all agents $i \in [N]$, $	\bigcap_{j \geq{\tau}} S_j^{(i)} = S_{{\tau}}^{(i)}$.
\end{proof}

In other words, after phase ${\tau}$, the system is \emph{frozen}, i.e., the set of arms of all agents remain fixed for \emph{all time in the future}. Moreover, all agents will only recommend the best arm going forward from this phase. We will show in the sequel that $\mathbb{E}[A_{\tau}] < \infty$ for all settings of the algorithm and hence the system freezes after only almost-surely finitely many changes in the set of arms played by the different agents. 

\subsection{Intermediate Propositions}

\begin{prop}
	The regret of any agent $i \in \{1,\cdots,N\}$ after playing for $T$ steps is bounded by 
	\begin{align*}
	\mathbb{E}[R_T^{(i)}] \leq \mathbb{E}[A_{\tau}] + \frac{K}{4} +  {4 \alpha \ln(T)}\left( \sum_{j=2}^{\lceil \frac{K}{N} \rceil + 2} \frac{1}{\Delta_j} \right) .
	\end{align*}
	\label{prop:weak_reg_decompose}
\end{prop}
	\begin{proof}

	From the definition of regret, we can write,
\begin{align*}
		R_T^{(i)} &= \sum_{l=1}^{T} (\mu_1 - \mu_{I_l^{(i)}}) ,\\
		&= \sum_{l=1}^{T} \sum_{j =2}^{K} \Delta_j \mathbf{1}_{I_l^{(i)} = j},\\
		&\leq A_{{\tau}} + \sum_{l =A_{ \tau} +1}^{T} \sum_{j =2}^{K} \Delta_j \mathbf{1}_{I_l^{(i)} = j}, \\
		&= A_{{\tau}} + \sum_{j=2}^{K} \Delta_j \sum_{l = A_{\tau} + 1}^{T} \mathbf{1}_{I_l^{(i)} = j}.
		\end{align*}
		Thus, taking expectations on both sides, we get that 
		\begin{align}
		\mathbb{E}[R_T^{(i)}] \leq \mathbb{E}[ A_{{\tau}} ] + \sum_{j =2}^{K} \Delta_j \sum_{l=A_{{\tau}} +1}^{T} \mathbb{P}[I_l^{(i)} = j,j \in S_{{\tau}}^{(i)}].
		\label{eqn:arm_mean_1}
		\end{align}
		We can break up the summation on the RHS as follows. Fix an arm $j \in \{2,\cdots,K\}$ and evaluate the sum
		\begin{align}
		\sum_{l=A_{ {\tau}} +1}^{T}  \mathbb{P}[I_l^{(i)} = j,j \in S_{{\tau}}^{(i)}] &= \sum_{l=A_{{\tau}} +1}^{T} \mathbb{P}\left[I_l^{(i)} = j, T_j^{(i)}(l) \leq \frac{4 \alpha \ln(T)}{\Delta^2_j},j \in S_{{\tau}}^{(i)} \right] + \\ &\sum_{l=A_{ {\tau}} +1}^{T} \mathbb{P}\left[I_l^{(i)} = j, T_j^{(i)}(l) \geq \frac{4 \alpha \ln(T)}{\Delta^2_j},j \in S_{{\tau}}^{(i)} \right] , \nonumber \\
		&\leq \frac{4 \alpha \ln(T)}{\Delta^2_j}  \mathbb{P}[j \in S_{{\tau}}^{(i)}] + \sum_{l=A_{ {\tau}} +1}^{T}  \mathbb{P}\left[I_l^{(i)} = j, T_j^{(i)}(l) \geq \frac{4 \alpha \ln(T)}{\Delta^2_j} \right], \nonumber \\
		&\leq \frac{4 \alpha \ln(T)}{\Delta^2_j}  \mathbb{P}[j \in S_{{\tau}}^{(i)}] + \sum_{l=3}^{\infty} 2l^{2(1-\alpha)}, 
		\label{eqn:arm_mean_2}
		\end{align}
		where in the last line we substitute the classical estimate from \citep{ucb_auer}. We can use this estimate, as we know that both the best arm, i.e., arm indexed $1$ and the sub-optimal arm indexed $j$ are in the set $S_{{\tau}}^{(i)}$ and hence the agent can potentially play those arms. Now plugging Equation (\ref{eqn:arm_mean_2}) into Equation (\ref{eqn:arm_mean_1}), we get that 
		\begin{align*}
		\mathbb{E}[R_T^{(i)}] &\leq \mathbb{E}[ A_{{\tau}} ] +  \sum_{j =2}^{K} \Delta_j \left( \frac{4 \alpha \ln(T)}{\Delta_j^2}  \mathbb{P}[j \in S_{{\tau}}^{(i)}] + \sum_{l=3}^{\infty} 2l^{2(1-\alpha)}\right) ,\\
		&\stackrel{(a)}{\leq} \mathbb{E}[ A_{{\tau}} ] +   \frac{4 \alpha \ln(T)}{\Delta}\sum_{j =2}^{K}   \mathbb{P}[j \in S_{{\tau}}^{(i)}]  + \sum_{j =2}^{K} \frac{\Delta_j}{4} ,\\
		&\stackrel{(b)}{\leq} \mathbb{E}[ A_{{\tau}} ] +   \frac{4 \alpha \ln(T)}{\Delta} \left( \bigg \lceil \frac{K}{N} \bigg\rceil + 2 \right) + \frac{K}{4}.
		\end{align*}
		In step $(a)$, we use the bound that $\Delta_j \geq \Delta$, for all $j \in \{2,\cdots,K\}$ and the fact that for $\alpha > 3$, we have $\sum_{l=3}^{\infty} 2l^{2(1-\alpha)} \leq 1/8$. In step $(b)$, we use the crucial identity that for any agent $i \in \{1,\cdots,N\}$ and any phase $\psi$ either deterministic or random, we have almost-surely,
		\begin{align*}
		\sum_{j = 1}^{K} \mathbf{1}_{j \in S_{\psi}^{(i)}} = \bigg\lceil \frac{K}{N} \bigg\rceil + 2.
		\end{align*}

		Taking expectations on both sides yields the result. If one were more precise in step $(a)$, then it is possible to establish that 
		\begin{align*}
		\mathbb{E}[R_T^{(i)}] &\leq  \mathbb{E}[ A_{{\tau}} ] +  {4 \alpha \ln(T)}\sum_{j =2}^{K}  \frac{1}{\Delta_j} \mathbb{P}[j \in S_{{\tau}}^{(i)}]  + \sum_{j =2}^{K} \frac{\Delta_j}{4}, \\
		&\leq \mathbb{E}[ A_{{\tau}} ] +   {4 \alpha \ln(T)}\left( \sum_{j=2}^{\lceil \frac{K}{N} \rceil + 2} \frac{1}{\Delta_j} \right) + \sum_{j =2}^{K} \frac{\Delta_j}{4}.
		\end{align*}
		This will then yield the proof. 
	\end{proof}

\begin{prop}
	For all $N \in \mathbb{N}$, $\Delta \in (0,1]$, $\alpha > 3$ and $M > 0$,
	\begin{align*}
\mathbb{E}[A_{{\tau}}] \leq A_{j^{*}} +  \frac{2}{2\alpha-3}\sum_{l \geq \frac{j^{*}}{2}-1} \frac{A_{2l+1}}{A_{l-1}^{3}}+ \mathbb{E}[A_{2 \widehat{\tau}_{spr}}],
	\end{align*}
	where $j^{*}$ is defined in Theorem \ref{thm:strong_result}.
	\label{prop:strong_cost}
\end{prop}
\begin{proof}
	Recall the fact that for any $\mathbb{N}$ valued random variable $X$, its expectation can be written as a sum of its tail probabilities, i.e., $\mathbb{E}[X] = \sum_{t \geq 1}\mathbb{P}[X \geq t]$. We use this fact to bound the expected value of $\mathbb{E}[A_{{\tau}}]$ as 
	\begin{align*}
	\mathbb{E}[A_{{\tau}}] 
	&= \sum_{t \geq 1} \mathbb{P}[  A_{\tau } \geq t], \\
	&\stackrel{(a)}{\leq} \sum_{t \geq 1} \mathbb{P}[  \tau  \geq A^{-1}(t)], \\
	&= \sum_{t \geq 1} \mathbb{P}[  \widehat{\tau}_{stab} + \widehat{\tau}_{spr} \geq  A^{-1}(t)] ,\\
	&\leq \sum_{t \geq 1} \mathbb{P}\left[  \widehat{\tau}_{stab}  \geq \frac{1}{2}(A^{-1}(t)) \right] + \sum_{t \geq 1} \mathbb{P} \left[\widehat{\tau}_{spr} \geq \frac{1}{2}\left( A^{-1}(t)\right)\right],\\
	&\leq A_{j^{*}} + \sum_{t \geq A_{j^{*}}+1 }  \mathbb{P}\left[  \widehat{\tau}_{stab} \geq \frac{1}{2}\left(A^{-1}(t) \right)\right] + \mathbb{E}[A_{2 \widehat{\tau}_{spr}}].
	\end{align*}
	Step $(a)$ follows from the definition of $A^{-1}(\cdot)$ given in Theorem \ref{thm:strong_result}.
	The estimate for $\mathbb{E}[A(2 \widehat{\tau}_{spr})]$ follows by noticing that this random variable can be coupled to the spreading time for a classical rumor spreading model, which we do so in the sequel in Proposition \ref{prop:couple_spread}. The first summation can be bounded by using estimates from Lemma \ref{lem:error_estimate_ucb}. We do so by applying a  union bound over all agents and phases as follows. Fix some $x \geq j^{*}/2$ in the following calculations.
	\begin{align*}
	\mathbb{P}[\widehat{\tau}_{stab} \geq x] &= \mathbb{P}\left[ \bigcup_{i =1}^{N}\widehat{ \tau}_{stab}^{(i)} \geq x\right], \\
	&\leq \sum_{i=1}^{N} \mathbb{P}[ \widehat{\tau}_{stab}^{(i)} \geq x],\\
	&= \sum_{i=1}^{N} \mathbb{P} \left[ \bigcup_{l=x}^{\infty} \chi_l^{(i)} = 1 \right],\\
	&\leq \sum_{i=1}^{N}  \sum_{l \geq x} \mathbb{P}\left[  \chi_l^{(i)} = 1 \right],\\
	&\stackrel{(a)}{\leq} \sum_{i=1}^{N}  \sum_{l \geq x} \frac{2}{2\alpha-3}{K \choose 2} \left(\bigg \lceil \frac{K}{N} \bigg\rceil + 1\right) A_{l-1}^{-(2\alpha-3)},\\
	&=  \sum_{l \geq x} \frac{2}{2\alpha-3}N{K \choose 2} \left(\bigg \lceil \frac{K}{N} \bigg\rceil + 1\right) A_{l-1}^{-(2\alpha-3)},\\
	& \stackrel{(b)}{\leq} \frac{2}{2\alpha-3}  \sum_{l \geq x} A_{l-1}^{-3},\\
	\end{align*}
	In the above calculations, we use the bound from Lemma \ref{lem:error_estimate_ucb} in step $(a)$ as $x \geq j^{*}/2$. In step $(b)$, we use $N{K \choose 2} \left(\bigg \lceil \frac{K}{N} \bigg\rceil + 1\right)  \leq \left(A_{\frac{j^{*}}{2}-1}\right)^{2\alpha-6}$, which follows from the definition of $j^{*}$ given in Theorem \ref{thm:strong_result}. Thus, we can obtain the following.

		\begin{align*}
	\sum_{t \geq A_{j^{*}}+1 }  \mathbb{P}\left[  \widehat{\tau}_{stab} \geq \frac{1}{2}\left(A^{-1}(t)\right) \right] &\leq \sum_{t \geq A_{j^{*}}+1 }   \left(\frac{2}{2\alpha-3}\right) \sum_{l \geq \frac{1}{2} A^{-1}(t) } A_{l-1}^{-3},\\
&\leq  \left(\frac{2}{2\alpha-3}\right)\sum_{t \geq A_{j^{*}}+1 }   \sum_{l \geq \frac{1}{2} A^{-1}(t)} A_{l-1}^{-3},\\
	&\stackrel{(c)}{\leq} \left(\frac{2}{2\alpha-3}\right)\sum_{l \geq \frac{1}{2}A^{-1}(A_{j^{*}}+1)} \sum_{t = A_{j^{*}}+1}^{A_{2l}} A_{l-1}^{-3},\\
	&\stackrel{}{\leq} \left(\frac{2}{2\alpha-3}\right)\sum_{l \geq \frac{j^{*}}{2}} \frac{A_{2l}}{A_{l-1}^{3}} < \infty.\\
	\end{align*}
	Step $(c)$ follows by swapping the order of summations. The condition {\bf A.2} in Section \ref{sec:assumptions} satisfied by the sequence $(A_j)_{j \in \mathbb{N}}$ ensures that the last summation is finite.

	\end{proof}

\begin{prop}
	The random variable $\widehat{\tau}_{spr}$ is stochastically dominated by $\tau_{spr}^{(P)}$.
		\label{prop:couple_spread}
\end{prop}
\begin{proof}
	We construct a coupling of the spreading process induced by our algorithm and a PULL based rumor spreading on $P$. We construct the coupling as follows. First we sample the reward vectors 
	$({Y}_j^{(i)}(t))_{i \in [N], j \in [K], t \geq 0}$. Then we can construct the random variable $\widehat{\tau}_{stab}$, which is a measurable function of the reward vectors. We then sample the communication random variables of our algorithm $(Z_j^{(i)})_{i \in [N], j \geq 0}$. We then construct a PULL based communication protocol with the random variables $(Z_j^{(i)})_{i \in [N], j\geq \widehat{\tau}_{stab}}$. Since $\widehat{\tau}_{stab}$ is independent of $(Z_j^{(i)})_{i \in [N], j \geq 0}$, the sequence of $(Z_{j -\widehat{\tau}_{stab}}^{(i)})_{i \in [N], j\geq \widehat{\tau}_{stab}}$ is identically distributed as $(Z_j^{(i)})_{i \in [N], j \geq 0}$. 
	\\
	
	Now, for the stochastic domination, consider the case where in the PULL based system, which starts at phase (time) $\widehat{\tau}_{stab}$, only agent $1$ has the rumor (best-arm). By definition of $\widehat{\tau}_{stab}$, any agent that contacts another agent possesing the rumor (best-arm), is also aware of the rumor (best-arm). The stochastic domination is concluded as at phase $\widehat{\tau}_{stab}$, many agents may be aware of the rumor (best-arm) in our algorithm, while in the rumor spreading process, only agent $1$ is aware of the rumor at phase $\widehat{\tau}$.
	
\end{proof}

\subsection*{Proof of  Theorem \ref{thm:strong_result}}
\begin{proof}
We can conclude Theorem \ref{thm:strong_result} by plugging in the estimates from Propositions \ref{prop:strong_cost} and \ref{prop:couple_spread} into Proposition \ref{prop:weak_reg_decompose}. 
\end{proof}

\section{Analysis of the UCB Error Estimates}


\begin{lemma}
	For any agent $i \in [N]$ and phase $j$ such that  $\frac{A_j - A_{j-1}}{2 + \lceil \frac{K}{n}\rceil } \geq 1 + \frac{4 \alpha \log(A_j)}{\Delta^2}$, we have 
	\begin{align*}
	\mathbb{E}[\chi_j^{(i)}] \leq \frac{2}{2\alpha-3}{K \choose 2} \left( \bigg \lceil \frac{K}{N} \bigg\rceil+1\right)  \left(\frac{1}{A_{j-1}^{2\alpha - 3}}  \right),
	\end{align*}
	where $\chi_j^{(i)}$ is defined in Equation (\ref{eqn:chi_defn_determinisstic}).
	\label{lem:error_estimate_ucb}
\end{lemma}

\begin{proof}
	As the algorithm recommends the most played arm in a phase, the arm that is recommended (i.e., $\mathcal{O}_j^{(i)}$) must be payed by agent $i$ at-least $\frac{A_{j} - A_{j-1}}{|S_j^{(i)}|}$ times in phase $j$. This follows from an elementary pigeon hole argument. Let $\mathcal{S}^{(i)}$ be the collection of all subsets $S \subset \{1,\cdots,K\}$ such that $\widehat{S}^{(i)} \subset S$ and $1 \in S$. Let $\mathcal{A}_j$ be a collection of all $\mathbb{N}$ valued tuples $(a_1,\cdots a_{\lceil \frac{K}{N} \rceil +2}) \in \mathbb{N} \text{ s.t. } \sum_{m =0}^{\lceil \frac{K}{N} \rceil +2} a_m = A_{j-1}$. We shall however, consider all possible histories, i.e., $\mathbb{N}^{\lceil \frac{K}{N} \rceil +2}$.

	\begin{align}
	\mathbb{E}[\chi_j^{(i)}] &\stackrel{(a)}{\leq} \sum_{S \in \mathcal{S}^{(i)}} \mathbb{P} \left[ \bigcup_{{(a_1,\cdots a_{\lceil \frac{K}{N} \rceil +2})\in \mathbb{N}^{\lceil \frac{K}{N} \rceil +2}}}   \chi_j^{(i)}(S; a_1,\cdots,a_{\lceil \frac{K}{N} \rceil}) \right], \nonumber\\
	&\overset{(b)}{\leq} \sum_{S \in \mathcal{S}^{(i)}} \mathbb{P}\left[  \bigcup_{{(a_1,\cdots a_{\lceil \frac{K}{N} \rceil +2})\in \mathbb{N}^{\lceil \frac{K}{N} \rceil +2}}}\bigcup_{l \in S, l\neq 1} T_l^{(i)}(A_j) - T_l^{(i)}(A_{j-1}) \geq \frac{A_{j} - A_{j-1}}{|S|} \right],\nonumber\\
		&\overset{(c)}{\leq}\sum_{S \in \mathcal{S}^{(i)}} \sum_{t= A_{j-1} + \frac{A_{j} - A_{j-1}}{|S_j^{(i)}|}}^{A_{j}}  \mathbb{P}\left[  \bigcup_{{(a_1,\cdots a_{\lceil \frac{K}{N} \rceil +2})\in \mathbb{N}^{\lceil \frac{K}{N} \rceil +2}}} \bigcup_{l \in S, l\neq 1 } T_{l}^{(i)}(t-1) - T_{l}^{(i)}(A_{j-1}) = \frac{A_{j} - A_{j-1}}{|S_j^{(i)}|} - 1, {I}_{t}^{(i)} = l\right], \nonumber\\
		&\overset{(d)}{\leq}\sum_{S \in \mathcal{S}^{(i)}}  \sum_{t= A_{j-1} + \frac{A_{j} - A_{j-1}}{|S_j^{(i)}|}}^{A_{j}} \sum_{l \in S, l \neq 1} \mathbb{P}\left[  \bigcup_{{(a_1,\cdots a_{\lceil \frac{K}{N} \rceil +2})\in \mathbb{N}^{\lceil \frac{K}{N} \rceil +2}}}    T_{l}^{(i)}(t-1) - T_{l}^{(i)}(A_{j-1}) = \frac{A_{j} - A_{j-1}}{|S_j^{(i)}|} - 1, {I}_{t}^{(i)} = l\right], \nonumber\\
			&\overset{(e)}{\leq} \sum_{S \in \mathcal{S}^{(i)}}\sum_{ l \in S, l \neq 1}  \sum_{t= A_{j-1} + \frac{A_{j} - A_{j-1}}{|S_j^{(i)}|}}^{A_{j}} \mathbb{P}\left[   T_{l}^{(i)}(t-1) \geq \frac{A_{j} - A_{j-1}}{|S_j^{(i)}|} - 1,
			\text{UCB}_l^{(i)}(t) \geq \text{UCB}_1^{(i)}(t)
			 \right] ,\label{eqn:ucb_lem_holder}\\
			&\overset{(f)}{\leq} \sum_{S \in \mathcal{S}^{(i)}}\sum_{ l \in S, l \neq 1}  \sum_{t= A_{j-1} + \frac{A_{j} - A_{j-1}}{|S_j^{(i)}|}}^{A_{j}}2t^{2(1-\alpha)} \label{eqn:ucb_proof1}.
	\end{align}
	Step $(a)$ follows from an union bound over $\mathcal{S}^{(i)}$.
	In step $(b)$ we use the fact that if an arm $l$ has to be the most played, then it must be played at-least $\frac{A_{j} - A_{j-1}}{|S|}$ times. In step $(c)$, we search over times, when the number of times arm $l$ has been played exceeds $\frac{A_{j} - A_{j-1}}{|S|}$ exactly. In step $(d)$, we use an union bound over $S$. In step $(e)$, for any arm $l \in S$, $\text{UCB}_l^{(i)}(t) = \widehat{\mu}_l^{(i)}(t-1) + \sqrt{\frac{\alpha \ln(t)}{T_l^{(i)}(t-1)}}$. In step $(e)$, we ask that arm $l$ and $1$ has been played at-least $0$ or more times in the past before time $t$ and that the UCB index of arm $l$ at agent $i$ at time $t$, exceed that of the index of the best arm. In step $(f)$, we plug in the classical estimate from \citep{ucb_auer}. This bound is applicable in our case as $1 \in S$ and the arm gap between the best and the second best arm in $S$ is at-least $\Delta$. Furthermore, the condition in the lemma $\frac{A_j - A_{j-1}}{2 + \lceil \frac{K}{n}\rceil } \geq 1 + \frac{4 \alpha \log(A_j)}{\Delta^2}$ implies that for all $t \in \left[A_{j-1} + \frac{A_{j} - A_{j-1}}{|S_j^{(i)}|}, A_j \right]$, the conditions in the bound in \citep{ucb_auer} is satisfied and is hence applicable. Notice that $|\mathcal{S}^{(i)}| \leq {K \choose 2}$. Thus, switching the order of summation and simplifying Equation \ref{eqn:ucb_proof1}, we get
	\begin{align*}
		\mathbb{E}[\chi_j^{(i)}] &\leq {K \choose 2} \left( \bigg \lceil \frac{K}{N} \bigg\rceil+1\right)  \sum_{t= A_{j-1} + \frac{A_{j} - A_{j-1}}{|S_j^{(i)}|}}^{A_{j}} 2t^{2(1-\alpha)} \\
		&\leq  2{K \choose 2} \left( \bigg \lceil \frac{K}{N} \bigg\rceil+1\right) \int_{A_{j-1}}^{A_j} u^{2(1- \alpha)} \mathrm{d}u \\
		&\leq \frac{2}{2\alpha-3}{K \choose 2} \left( \bigg \lceil \frac{K}{N} \bigg\rceil+1\right)  \left(\frac{1}{A_{j-1}^{2\alpha - 3}} -\frac{1}{A_{j}^{2\alpha - 3}} \right) .
	\end{align*}
\end{proof}
	
	Similarly, we also have a bound on the error probability in the case of random phase length system in the following lemma.
	\begin{lemma}
	\label{lem:error_estimate_ucb_random}
	For any agent $i \in [N]$ and every $j \in \mathbb{N}$ such that  $\frac{A_j - A_{j-1}}{2 + \lceil \frac{K}{n}\rceil } \geq 1 + \frac{4 \alpha \log(A_j)}{\Delta^2}$, we have 
	\begin{align*}
	\mathbb{E}[\chi_j^{(i)} \mathbf{1}_{j \geq H^{*}}] \leq \frac{2}{2\alpha-3}{K \choose 2} \left( \bigg \lceil \frac{K}{N} \bigg\rceil+1\right)  \left(\frac{1}{A_{j-1}^{2\alpha - 3}}  \right),
	\end{align*}
	where $\chi_j^{(i)}$ is defined in Equation (\ref{eqn:chi_defn_random}).
\end{lemma}
\begin{proof}
    The proof is identical to that in Lemma \ref{lem:error_estimate_ucb} upto Equation \ref{eqn:ucb_lem_holder}, where the upper limit of summation is $(1+\delta)A_j$ in the asynchronous communication scenario. Continuing with the rest of the calculation, identical to that in Lemma \ref{lem:error_estimate_ucb} yields the result.
\end{proof}

\section {Poisson Asynchronous Algorithm - Buildup to Proof of Theorem \ref{thm:prob_main_result}}
\label{sec:poisson_asynch_algo}
In order to prove Theorem \ref{thm:prob_main_result}, we will state a more general algorithm in the sequel in Algorithm \ref{algo:poisson-algo} and prove a performance bound on it in Theorem \ref{thm:poisson_result}. We shall then subsequently prove Theorem \ref{thm:poisson_result} in Appendix \ref{sec:proof_poisson_result} and as a corollary of the proof, deduce Theorem \ref{thm:prob_main_result} in Appendix \ref{sec:prob_algo_proof}.

\begin{algorithm}
	\caption{Distributed Poisson Asynchronous MAB Regret Minimization (at Agent $i$)}
	\begin{algorithmic}[1]
		\State \textbf{Input Parameters}: Communication Budget $(B_t)_{t \in \mathbb{N}}, \text{ UCB Parameter }\alpha, \text{ Slack }\delta$, $\varepsilon > 0$
		\State \textbf{Initialization}:  $ \widehat{S}^{(i)},S_i^{(0)}$ according to Equations (\ref{eqn:hat_S_i}) and (\ref{eqn:S_i_0}).
		
		\State $j\gets 0$ 
        \State $A_j = \max \left(\inf \{t \geq 0, B_t \geq j\}, (1+j)^{1+\varepsilon}\right)$ \Comment{Reparametrize the commuication budget}
        \State $\mathcal{P}_j \sim \text{Poisson}\left(\left(1+\frac{\delta}{2}\right)(A_j-A_{j-1})\right)$ \Comment{Poisson distributed phase length}
		\For{Time $t \in \mathbb{N}$}
		\State Pull arm - $\arg\max_{l \in S_i^{(j)}} \left( 
		\widehat{\mu}_l^{(i)}(t-1) + \sqrt{\frac{\alpha \ln(t)}{T_l^{(i)}(t-1)}}\right)$
		
		\If {$t == \sum_{y=0}^{j}\mathcal{P}_{y}$}
		\State $O_i^{(j)} \gets$ {\ttfamily GET-ARM-PREV}($i,t$) \Comment{Given in Algorithm \ref{algo:get-prev-arm}}
		\If {$O_i^{(j)} \not \in S_i^{(j)}$}
		\State $U_{j+1}^{(i)} \gets \arg\max_{l \in \{U_j^{(i)},L_j^{(i)}\}}( T_l(A_j)- T_l(A_{j-1} ))$ \Comment{The most played arm}
		\State $L_{j+1}^{(i)} \gets O_j^{(i)}$ \Comment{Update the set of playing arms}
		\State $S_{j+1}^{(i)} \gets \widehat{S}^{(i)} \cup L_{j+1}^{(i)} \cup U_{j+1}^{(i)}$
		\Else
		\State $S_{j+1}^{(i)} \gets S_j^{(i)}$.
		\EndIf 
		\State $j \gets j+1$ 
\State $A_j = \max \left(\inf \{t \geq 0, B_t \geq j\}, (1+j)^{1+\varepsilon}\right)$ \Comment{Reparametrize the commuication budget}
        \State $\mathcal{P}_j \sim \text{Poisson}\left(\left(1+\frac{\delta}{2}\right)(A_j-A_{j-1})\right)$
		\EndIf
		\EndFor
	\end{algorithmic}
	\label{algo:poisson-algo}
\end{algorithm}

This algorithm does not fit our framework exactly, as the communication budget is not necessarily met. In particular, this algorithm only ensures that \emph{with high probability}, the number of information pulls by agents in the first $t$ time slots is within the prescribed budget $B_t$. Thus, we present this algorithm in the Appendix and not as a solution to the multi-agent MAB problem. In order to prove this result, we will need a further assumption on the input parameters.
\\

\noindent {\bf (A.3) - } The communication budget $(B_t)_{t \in \mathbb{N}}$ and $\varepsilon > 0$ is such that $\exists \text{ } \kappa>0$, $\sum_{x \geq 1}A_{A_x }e^{-\kappa(A_{x} - A_{x-1})} < \infty$, where $(A_x)_{x \in \mathbb{N}}$ is given in Equation (\ref{eqn:A_x_defn}).
\\

\begin{theorem}

	Suppose in a system of $N \geq 2$ agents  connected by a communication matrix $P$ satisfying assumption {\bf A.1} and $K \geq 2$ arms, each agent runs Algorithm \ref{algo:poisson-algo}, with input parameters $(B_t)_{t \in \mathbb{N}}$, and the UCB parameter $\alpha > 3$ and $\varepsilon > 0$ satisfying assumptions {\bf A.2} and {\bf A.3} and $\delta >0$ such that $\exists D > 0$ with $c(\delta) \geq \frac{5}{4}D$ and $(3+2\delta + \ln(4+2\delta)) \geq \frac{5}{4}D^{-1}$, where $c(\delta) = \min  \left( \frac{\delta}{2} + \ln \left( 1 + \frac{\delta}{2} \right),  (1+\delta)\ln\left(\frac{2+2\delta}{2+\delta}\right)  - \frac{\delta}{2} \right)$. Then the regret of any agent $i \in [N]$, after any time $T \in \mathbb{N}$ is bounded by 

	\begin{equation*}
	\mathbb{E}[R_T^{(i)}] \leq \underbrace{\left( \sum_{j=2}^{\lceil \frac{K}{N} \rceil + 2}  \frac{1}{\Delta_j} \right) 4 \alpha \ln(T) + \frac{K}{4}}_{\text{{Collaborative UCB Regret}}} + \underbrace{ (1+\delta)\mathbb{E}[A_{2\lfloor 2+\delta\rfloor  \tau_{spr}^{(P)}}] + \widehat{g}_1((A_x)_{x \in \mathbb{N}},\delta) + N \widehat{g}_2((A_x)_{x \in \mathbb{N}},\delta)}_{\text{{Cost of Asynchronous Infrequent Pairwise Communications}}}, 
	\end{equation*}
	where 
	\begin{align*}
\widehat{g}_1((A_x)_{x \in \mathbb{N}},\delta) = 	2(1+\delta)\left(A_{2\lceil 2+\delta \rceil j^{*}} + \left(\frac{2}{2\alpha-3}\right)\sum_{l \geq3} \frac{A_{2l}}{A_{l-1}^{3}} + 2\sum_{x \geq  \lceil \frac{j^{*}}{2} \rceil} (A_{\lceil 4\lceil 2+\delta \rceil x \rceil})^{\left(2(2\alpha-6)+2\right)}e^{-c(\delta) (A_{x+1} - A_x)} \right),    
	\end{align*}
	and 
	\begin{multline*}
	  \widehat{g}_2((A_x)_{x \in \mathbb{N}},\delta) =  2 \left( A_{x_0}^2 e^{-c(\delta)(A_{x_0}-A_{x_0-1})}+\sum_{x \geq 1} A_x  e^{-c(\delta)A_x^{\omega}} \right) + \frac{1}{c(\delta)} + \\ (1+\delta)\left( 2\sum_{x \geq 1}A_{\lceil A_x \rceil}e^{-c(\delta)(A_{x} - A_{x-1})} + N \sum_{t \geq 1}e^{ -(3+2\delta + \ln (4+2\delta))A^{-1}(t) } \right),
	\end{multline*}

	where 	$j^{*}$ is given in Theorem \ref{thm:strong_result},
	and $x_0 \in \mathbb{N}$ is from Assumption {\bf A.2} in Section \ref{sec:assumptions}.
	\label{thm:poisson_result}
\end{theorem}
The proof of this theorem is carried out in Appendix \ref{sec:proof_poisson_result}.

\section{Proof of Theorem \ref{thm:poisson_result}}
\label{sec:proof_poisson_result}

For every agent $i \in [N]$ and phase $j \geq 0$, we shall denote by $\mathcal{P}_j^{(i)} \in \mathbb{N}$ to be the number of times agent $i$ pulls an arm in phase $j$. Notice from the conditions on the input parameter $(p_j)_{j \in \mathbb{N}}$ that the following property is satisfied - 
\begin{equation}
\begin{aligned}
\sum_{j \geq 0} \mathbb{P}[\mathcal{P}_j \leq A_j-A_{j-1}] < \infty, \\
\sum_{j \geq 0} \mathbb{P}[\mathcal{P}_j \geq (1+\delta)(A_j - A_{j-1})] < \infty.
\end{aligned}
\label{eqn:condn_pj}
\end{equation}

To make things simpler, we shall consider the following probability space. As before, it contains the reward and communication random variables $({Y}_j^{(i)}(t))_{i \in [N], j \in [K], t \geq 0}$ and $(Z_j^{(i)})_{j \geq 0, i\in [N]}$. For every $j \in [K]$, the collection $(Y_j^{(i)}(t))_{t \geq 0, i \in [N]}$ is an i.i.d. Bernoulli random variable of mean $\mu_j$. The interpretation being that if an agent $i$ pulls arm $j$ for the $l$th time, it will receive reward $Y_j^{(i)}(l)$. Similarly, for each $i \in [N]$, the sequence $(Z_j^{(i)})_{j \geq 0}$ is iid distributed as $P(i,\cdot)$. The interpretation is that when agent $i$ wishes to receive a recommendation at the end of phase $j$, it will do so from agent $Z_j^{(i)}$. In addition, we also assume that the probability space consists of another independent sequence $(\mathcal{P}_{j}^{(i)})_{i \in [N],j \geq 0}$, where for each $i \in [N]$ and $j \geq 0$, the random variable $\mathcal{P}_j^{(i)}$ is independent of everything else and distributed as a Poisson random variable with mean $\left( 1 + \frac{\delta}{2}\right)(A_{j+1}-A_j)$.

\subsection{Definition and Notations}

To proceed with the analysis, define by a $\mathbb{N}$ valued random variable $H^{*}$ as 
\begin{align*}
H^{*} =  \inf\left\{ j^{'} \geq 0 : \forall i \in [1,N], \forall j \geq j^{'},   \mathcal{P}_j^{(i)} \in [A_j-A_{j-1}, (1+\delta)(A_j-A_{j-1})]  \right\},
\end{align*}
 Equations (\ref{eqn:condn_pj}) imply from Borel Cantelli lemma that $H^{*} < \infty$ almost-surely. We will need  another random variable $\Gamma\in \mathbb{N}$, which is defined as

\begin{align*}
\Gamma = \sup \left\{t \geq 0: \exists i \in [N], \sum_{j=0}^{H^{*}}\mathcal{P}_{j}^{(i)} \geq t \right\}.
\end{align*}
In words, $\Gamma$ represents the time when the last agent shift to phase $H^{*}$. 
Similar to that done in the proof of Theorem \ref{thm:strong_result}, we define a sequence of indicator random variables $(\chi_j^{(i)})_{j \geq 0,i\in[N]}$ as follows. The definition is identical to the one used in the proof of Theorem \ref{thm:strong_result}, which we reproduce here for completeness. Fix some agent $i \in [N]$ and phase $j \geq 0$. Denote by $\mathcal{S}^{(i)}$ as the collection of all subsets $S$ of $[K]$ with cardinality $\lceil \frac{K}{N} \rceil + 2$, such that $\widehat{S}^{(i)} \subset S$ and $1 \in S$. Clearly, $|\mathcal{S}^{(i)}| \leq {K \choose 2}$.  Denote by the tuples $(a_1,\cdots,a_{\lceil \frac{K}{N} \rceil + 2}) \in \mathbb{N}^{\lceil \frac{K}{N} \rceil + 2}$ such that $\sum_{m=0}^{\lceil \frac{K}{N} \rceil + 2} a_m \geq 0$. For any set $S \in \mathcal{S}^{(i)}$ and tuples $(a_1,\cdots,a_{\lceil \frac{K}{N} \rceil + 2}) \in \mathcal{A}_j$, denote by the event $\xi_j^{(i)}(S, a_1,\cdots,a_{\lceil \frac{K}{N} \rceil + 2})$ as 
\begin{align*}
\xi_j^{(i)}(S, a_1,\cdots,a_{\lceil \frac{K}{N} \rceil + 2}) := \left\{ S_j^{(i)} = S, T_{l_1}(A_{j-1}) = a_1,\cdots,T_{l_{\lceil \frac{K}{N} \rceil + 2}}(A_{j-1}) =a_{\lceil \frac{K}{N} \rceil + 2}, \widehat{\mathcal{O}}_j^{(i)} \neq 1  \right\}.
\end{align*}
Denote by $\Xi_j^{(i)}$ as the union of all such events above, i.e., 
\begin{align*}
\Xi_j^{(i)} := \bigcup_{S \in \mathcal{S}^{(i)}} \left( \bigcup_{(a_1,\cdots,a_{\lceil \frac{K}{N} \rceil + 2}) \in \mathbb{N}^{\lceil \frac{K}{N} \rceil + 2}} \xi_j^{(i)}(S, a_1,\cdots,a_{\lceil \frac{K}{N} \rceil + 2})\right).
\end{align*}
Denote by $\chi_j^{(i)}$ as the indicator random variable, i.e.,
\begin{align}
\chi_j^{(i)} = \mathbf{1}_{\Xi_j^{(i)}}.
\label{eqn:chi_defn_random}
\end{align}
Observe that, as before, for all agents $i \in [N]$ and phases $j \geq 0$, the random variable $\xi_j^{(i)}$ is measurable with respect to the reward sequence $(Y_l^{(i)})_{l \leq A_j}$. Furthermore, we have the almost-sure inequality that 
\begin{align*}
\xi_j^{(i)} \geq \mathbf{1}_{1 \in S_j^{(i)}, \widehat{O}_j^{(i)} \neq 1, j \geq H^{*}}.
\end{align*}
This follows from the same reasoning as in Theorem \ref{thm:strong_result} as $\xi_j^{(i)}$ considers all possible sample paths for communication while $\mathbf{1}_{1 \in S_j^{(i)}, \widehat{O}_j^{(i)} \neq 1, j \geq H^{*}}$ is for a particular sample path of communications among agents. Notice that since the phase lengths are random, we can only reason about the sample path for agent phases larger than or equal to $H^{*}$. 
\\

Similar to before, we define the random variables $\widehat{\tau}_{stab}^{(i)}, \widehat{\tau}_{stab}, \widehat{\tau}_{spr}^{(i)}$ and $\widehat{\tau}_{spr}$. Denote by $\tau = \widehat{\tau}_{stab} + \widehat{\tau}_{spr}$. These definitions from the Proof of Theorem \ref{thm:strong_result} are reproduced here for completeness.

\begin{align*}
\widehat{\tau}_{stab}^{(i)} &= \inf \{ j^{'} \geq j^{*} : \forall j \geq j^{'}, \chi_j^{(i)} = 0 \}, \\
\widehat{\tau}_{stab} &= \max_{i \in [N]} \widehat{\tau}_{stab}^{(i)},\\
\widehat{\tau}_{spr}^{(i)} &= \inf\{j \geq \widehat{\tau}_{stab} : 1 \in S_j^{(i)} \} - \widehat{\tau}_{stab}, \\
\widehat{\tau}_{spr} &= \max_{i \in \{1,\cdots,N\}} \widehat{\tau}_{spr}^{(i)}, \\
{\tau} &= \widehat{\tau}_{stab}+\widehat{\tau}_{spr}.
\end{align*}

From the definitions, the statement and proof of Proposition \ref{prop:strong_freeze} holds verbatim for the present algorithm as well. We will need two additional definitions to help state our result. Denote by $T_{stab} \in \mathbb{N}$ to be the first time when all agents pull arms and are in phase $\widehat{\tau}_{stab}$ or larger, i.e., 
\begin{align*}
T_{stab} = \sup \left\{t \geq \Gamma : \exists i \in [N], \sum_{j=0}^{\widehat{\tau}_{stab}} \mathcal{P}_j^{(i)} \geq t \right\}.
\end{align*}
Similarly, define $H$ to be the maximum over all agents phases at time $T_{stab}$, i.e., 
\begin{align*}
H = \sup \left\{ j \geq 0 : \exists i \in [N], \sum_{l=0}^{j} \mathcal{P}_l^{(i)} \leq T_{stab}\right\}.
\end{align*}
Similarly, denote by $\mathcal{T}$ as the first time when all agents pull arms in phase $\tau$ or larger, i.e.,
\begin{align*}
\mathcal{T} = \sup \left\{t \geq 0 : \exists i \in [N], \sum_{j=0}^{{\tau}} \mathcal{P}_j^{(i)} \geq t \right\}
\end{align*}

\subsection{Structural Results}
In this section, we give inequalities relating the random variables defined in the previous section, that will be helpful in proving Theorem \ref{thm:prob_main_result}.
\begin{lemma}
	\begin{align*}
	\mathbb{E}[\mathcal{T}] \leq 
	\mathbb{E}[ T_{stab}] + (1+\delta)(\mathbb{E}[A_{H+(2\tau_{spr}^{(P)}-1)\lfloor 2+\delta \rfloor + 1} - A_{H}])
	\end{align*}
	where the random variable $\tau_{spr}^{(P)}$ is independent of $H$. 
	\label{lem:stoch_dom}
\end{lemma}
\begin{proof}
	
	The proof consists of three steps. First, we will  construct a coupling with a standard PULL based rumor process on the communication matrix $P$ such that $H$ and $\tau_{spr}^{(P)}$ are independent. Then we shall argue a stochastic domination and for the constructed coupling show that, almost-surely, we have 
	\begin{align}
	\mathcal{T} \leq	 T_{stab} + (1+\delta)(A_{H+(\tau_{spr}^{(P)}-1)\lfloor 2+\delta \rfloor + 1} - A_{H})
 \label{eqn:stoch_dom}
	\end{align}
	where $\tau_{spr}^{(P)}$ is independent of $H$ and $\leq_{st}$ represents stochastic domination. This will then conclude the proof by taking expectations on both sides. 
	\\
	
	{\bf (1) Coupling Construction }-
	We proceed with the coupling as follows. We assume that our probability space consists of the random variables $(Y_j^{(i)})_{i \in [N], j\in[K], l \geq 0}, (\mathcal{P}_{j}^{(i)})_{j \geq 0,i\in[N]}, (Z_j^{(i)})_{j \geq 0,i\in[N]}$ and $(\widehat{Z}_{j}^{(i)})_{j \geq 0,i\in[N]}$. The sequence $(Y_j^{(i)})_{i \in [N], j\in[K], l \geq 0}$ is independent of everything else and is used to construct the observed rewards of agents. The sequence  $(\mathcal{P}_{j}^{(i)})_{j \geq 0,i\in[N]}$ is independent of everything else and denotes the phase length random variables of agents as before. The sequence $(\widehat{Z}_{j}^{(i)})_{j \geq 0,i\in[N]}$ denotes a standard PULL based rumor spreading process on $P$, independent of everything else. In other words, for each agent $i \in [N]$, the sequence $(\widehat{Z}_j^{(i)})_{j \geq 0}$ is i.i.d., with each element distributed according to the distribution $P(i,\cdot)$. Thus, they represent the sequence of callers called by agent $i$ in the PULL based rumor process. The random communication sequence $(Z_j^{(i)})_{j \geq 0,i\in[N]}$ will be constructed such that it is independent of $(Y_j^{(i)})_{i \in [N], j\in[K], l \geq 0}, (\mathcal{P}_{j}^{(i)})_{j \geq 0,i\in[N]}$, and  equal in distribution to $(\widehat{Z}_{j}^{(i)})_{j \geq 0,i\in[N]}$ such that the stochastic domination in Equation (\ref{eqn:stoch_dom}) holds.
	\\
	
	
	 To do so, we will recursively define a sequence of random times $(t_i)_{i \geq 0}$ which are measurable with respect to the agent rewards and phases, i.e., for all $i \geq 0$, $t_i \in \sigma ((Y_j^{(i)}(l))_{i \in [N],j\in [K],l\geq 0 },(\mathcal{P}_j^{(i)})_{j \geq 0, i\in [N]} )$. Let $t_0 = T_{stab}$. We know that $T_{stab}$ is measurable only with respect to the reward random variables $(Y_j^{(i)}(l))_{i \in [N],j\in [K],l\geq 0 }$ and the phase lengths of the agents $(\mathcal{P}_j^{(i)})_{j \geq 0, i\in [N]}$. For all $i \geq 1$, let $t_i$ be the first time after $t_{i-1}$, such that all agents have changed phase at-least once in the time interval $[t_{i-1},t_i]$. More formally, we have 
	\begin{align*}
	t_{i} = \inf \left\{ x > t_{i-1} :  \forall i \in [N],  \exists j \in \mathbb{N} \text { s.t.} \sum_{l=0}^{j} \mathcal{P}_l^{(i)} \geq t_{i-1}, \sum_{l=0}^{j+1} \mathcal{P}_l^{(i)} \leq x \right\}.
	\end{align*}
	We construct another sequence of random variables $(j_x^{(i)})_{x\geq 0,i\in[N]}$, where for every agent $i \in [N]$ and $x\geq 0$,  $j_{x}^{(i)}$ is the first phase change of agent $i$ in the time interval $[t_x, t_{x+1})$ of our algorithm, i.e.,
	\begin{align*}
	j_x^{(i)} = \inf \left\{ j \geq 0: \sum_{l=0}^{j-1} \mathcal{P}_l^{(i)} < t_x , \sum_{l=0}^{j} \mathcal{P}_l^{(i)} \geq t_x \right\},
	\end{align*} 
	 where $\sum_{l=0}^{-1} = 0$. By construction observe that for all agents $i \in [N]$ and all $x \geq 0$, the random variable $j_x^{(i)}$ is measurable with respect to the rewards and phase lengths.
	 \\
	 
	 Equipped with these definitions, we construct the communication random variables of our algorithm $(Z_j^{(i)})_{j \geq 0, i\in [N]}$ as follows. For every agent $i \in [N]$ and $x \geq 0$, we let 
	 \begin{align*}
	 Z_{j_{2x}^{(i)}}^{(i)} = \widehat{Z}_x^{(i)}.
	 \end{align*}
	 For an agent $i \in [N]$, and any phase $j \not \in \{j_x^{(i)},x\geq 0\}$, we let $Z_j^{(i)}$ be i.i.d., from $P(i,\cdot)$. 
	 \\
	 
	 We only look at alternate intervals $[t_0,t_1],[t_2,t_3]$ and so on because in our algorithm, an agent recommends the most played arm in the \emph{previous phase}. Thus, if an agent becomes aware of the best arm in interval say $[t_0,t_1]$, then it will definitely recommend it in phase $[t_2,t_3]$, if asked, as since $t_2 \geq \Gamma$, agent will recommend the best arm, and moreover at-least one phase elapses after the agent receives the best arm.
	 \\
	 
	 
	 {\bf (2) Stochastic Domination }-
 We now conclude about the stochastic domination as follows. In the algorithm, we will only consider even time intervals $[t_0,t_1],[t_2,t_3]$ and so on, where an agent becomes newly aware of the best arm. This is so since our recommendation algorithm only recommends the best arm in the previous phase. At time $t_0$, exactly one agent knows the rumor in the PULL rumor spreading process while potentially more agents may be aware of the rumor (best-arm) in the algorithm. Furthermore, we consider that there is exactly one communication request in the rumor spreading process per even time-interval, (i.e., in $[t_0,t_1), [t_2,t_3)$ and so on), while potentially many more can occur in our algorithm. Thus, we have the following almost-sure bound under the afore mentioned coupling, 
 \begin{align}
 \mathcal{T} \leq t_{2 \tau_{spr}^{(P)}}.
 \label{eqn:coupling_1}
 \end{align}
 
 {\bf (3) Deterministic Bounds on $(t_x)_{x \geq 0}$}-
 If we further establish that for all $x \geq 0$, almost-surely, we have 
 \begin{align}
 t_x \leq T_{stab} + (1+\delta)(A_{H+(x-1)\lfloor 2+\delta \rfloor + 1} - A_{H}),
 \label{eqn:couple_ub}
 \end{align}
 then we can conclude the proof from Equations (\ref{eqn:coupling_1}) and (\ref{eqn:couple_ub}). To establish Equation (\ref{eqn:couple_ub}), first observe that $T_{stab} \geq \Gamma$ almost-surely. Thus, if any agent $i \in [N]$ will be in any phase $j$, for at-least $A_{j}-A_{j-1}$ number of arm-pulls and for at-most $(1+\delta)(A_{j}-A_{j-1})$ number of arm-pulls. Thus, by definition at time $t_0$, we know that no agent is in phase $H+1$ or beyond. Thus, at time $t_0 + (1+\delta)(A_{H+1} - A_{H})$, we know that all agents would have changed phase at-least once after $t_0$. Thus, $t_1 \leq t_0 + (1+\delta)(A_{H+1} - A_{H})$ almost-surely. 
 \\
 
 We now make the above into an induction argument. For the base case, suppose that at time $t_0$, all agents are within phase $H$ $-$ which is true by definition. For all $0 \leq x^{'} \leq x$, assume the induction hypothesis that 
 \begin{align*}
 t_{x^{'}+1} \leq t_{x^{'}} + (1+\delta)(A_{H+x^{'}\lfloor 2+\delta \rfloor +1} - A_{H+x^{'}\lfloor 2+\delta \rfloor}),
 \end{align*}
 and that all agents at time $t_{x^{'}+1}$ are at phase $H+(x^{'}+1)\lfloor 2+\delta \rfloor$ or lower. Since $t_x \geq \Gamma$, we know that in the time interval  $[t_x,t_x + (1+\delta)(A_{H+x\lfloor 2+\delta \rfloor+1} - A_{H+x\lfloor 2+\delta \rfloor})]$, all agents would have changed phase at-least once. Thus, $t_{x+1} \leq t_x + (1+\delta)(A_{H+x\lfloor 2+\delta \rfloor+1} - A_{H+x\lfloor 2+\delta \rfloor})$. It now remains to conclude that all agents will be in phase $H+(x+1)\lfloor 2+\delta \rfloor$ or lower at time $t_x + (1+\delta)(A_{H+x\lfloor 2+\delta \rfloor+1} - A_{H+x\lfloor 2+\delta \rfloor})$. Notice that the maximum phase any agent can be in at time $t_x + (1+\delta)(A_{H+x\lfloor 2+\delta \rfloor+1} - A_{H+x\lfloor 2+\delta \rfloor})$, given that it was in a phase $H+x\lfloor 2+\delta \rfloor$ or lower at time $t_x$ is bounded above by Proposition \ref{prop:convex} as $H+x\lfloor 2+\delta \rfloor + \lfloor 2 + \delta \rfloor$. This then concludes the induction step and hence we have for all $x \geq 0$, almost-surely, by a simple telescoping sum
 \begin{align*}
 t_x &\leq t_0 + (1+\delta)\sum_{l=0}^{x-1}(A_{H+l \lfloor 2 + \delta \rfloor +1} - A_{H+l \lfloor 2 + \delta \rfloor}),\\
 &\leq t_0 + (1+\delta)(A_{H+(x-1) \lfloor 2 + \delta \rfloor +1} - A_H).
 \end{align*}

\end{proof}

\begin{lemma}
	For any agent $i \in [N]$, the regret after it has pulled arms for $T$ times is bounded by 
	\begin{align*}
	\mathbb{E}[R_T^{(i)}] \leq \mathbb{E}[\mathcal{T}] + \frac{K}{4} + 4 \alpha \left( \sum_{j=1}^{\lceil \frac{K}{N} \rceil +1} \frac{1}{\Delta_j} \right) \ln(T).
	\end{align*}
	\label{lem:prob_reg_decompose}
\end{lemma}
\begin{proof}
	The proof of this Lemma follows similarly to that of Lemma \ref{prop:weak_reg_decompose}. We can write the regret of any agent $i \in [N]$ as follows - 
	\begin{align*}
R_T^{(i)} &= \sum_{t=1}^{T} \mu_1 - \mu_{I_t^{(i)}},\\
& \leq \mathcal{T} + \sum_{t =\mathcal{T} +1}^{T} \mu_1 - \mu_{I_t^{(i)}},\\
&=\mathcal{T}+ \sum_{t=\mathcal{T} +1+1}^{T} \sum_{l=2}^{K}\Delta_l \mathbf{1}_{I_t^{(i)} = l},\\
&\stackrel{(a)}{= }\mathcal{T} +  \sum_{l =2}^{K}\Delta_l  \sum_{t = \mathcal{T}+1}^{T} \mathbf{1}_{I_t^{(i)} = l}\mathbf{1}_{l \in S_\tau^{(i)}}
\end{align*}
	In step $(a)$, we use Proposition \ref{prop:strong_freeze} that at time $\mathcal{T}$, all agents are in a phase that is at-least $\tau$. Furthermore, from Proposition \ref{prop:strong_freeze} (recall that the statement and proof of Proposition \ref{prop:strong_freeze} holds verbatim for the present case also) implies almost-surely that, for all $j \geq \tau$, and all $i \in [N]$, $S_j^{(i)} = S_{\tau}^{(i)}$. Taking expectations on the last display yields
\begin{align*}
\mathbb{E}[R_T^{(i)}] \leq \mathbb{E}[\mathcal{T}] +  \sum_{l=2}^{K} \Delta_l \sum_{t=\mathcal{T}+1}^{T} \mathbb{P}[I_t^{(i)}=l,l \in S_{\tau}^{(i)}]
\end{align*}
Using the same techniques as in the proof of Proposition \ref{prop:weak_reg_decompose}, i.e., following all steps from Equation \ref{eqn:arm_mean_2} onwards, one obtains 
\begin{align*}
\sum_{l=2}^{K} \Delta_l \sum_{t= \mathcal{T}+1}^{T} \mathbb{P}[I_t^{(i)}=l,l \in S_{\tau}^{(i)}] \leq  \left( \sum_{j=2}^{\lceil \frac{K}{N} \rceil + 2} \frac{1}{\Delta_j} \right) 4 \alpha \ln(T) + \frac{K}{4}
\end{align*}
\end{proof}

\begin{lemma}

	\begin{align*}
	\mathbb{E}[\mathcal{T}] \leq \mathbb{E}[\Gamma] + (1+\delta)\mathbb{E}[A_{\widehat{\tau}_{stab}}] +  (1+\delta)\mathbb{E}[A_{4\Gamma}] + (1+\delta)\mathbb{E}[A_{4\lceil 1+\delta \rceil \widehat{\tau}_{stab}}] +  (1+\delta)\mathbb{E}[A_{2\lfloor 2+\delta \rfloor\tau_{spr}^{(P)}}].
	\end{align*}
	\label{lem:T_prob}
\end{lemma}
\begin{proof}

From Lemma \ref{lem:stoch_dom}, we know that 
\begin{align*}
 	\mathbb{E}[\mathcal{T}] &\leq	 \mathbb{E}[T_{stab}] + (1+\delta)\mathbb{E}[(A_{H+(\tau_{spr}^{(P)}-1)\lfloor 2+\delta \rfloor + 1} - A_{H})], \\
 	&\stackrel{(a)}{\leq} \mathbb{E}[T_{stab}] + (1+\delta)(\mathbb{E}[A_{2H}] + \mathbb{E}[A_{2\lfloor 2+\delta \rfloor\tau_{spr}^{(P)}}]),\\
 	&\stackrel{(b)}{\leq} \mathbb{E}[\Gamma] + (1+\delta)\mathbb{E}[A_{\widehat{\tau}_{stab}}] + (1+\delta) \mathbb{E}[A_{2\Gamma + 2\lceil 1+\delta \rceil \widehat{\tau}_{stab}}] + (1+\delta)\mathbb{E}[A_{2\lfloor 2+\delta \rfloor\tau_{spr}^{(P)}}], \\
 	&\stackrel{(c)}{\leq} \mathbb{E}[\Gamma] + (1+\delta)\mathbb{E}[A_{\widehat{\tau}_{stab}}]  + (1+\delta)\mathbb{E}[A_{4\Gamma}] + (1+\delta)\mathbb{E}[A_{4\lceil 1+\delta \rceil \widehat{\tau}_{stab}}] +  (1+\delta)\mathbb{E}[A_{2\lfloor 2+\delta \rfloor\tau_{spr}^{(P)}}].
\end{align*}
Steps $(a)$ and $(c)$ follow from the elementary fact that for any two random variables $X$ and $Y$ and any invertible function $f(\cdot)$, $\mathbb{E}[f(X+Y)] \leq \mathbb{E}[f(2X)] + \mathbb{E}[f(2Y)]$. Step $(b)$ follows from Lemma \ref{lem:T_stab}.

\end{proof}

\begin{lemma}
	\begin{align*}
	\mathbb{E}[T_{stab}] &\leq \mathbb{E}[\Gamma] + (1+\delta)\mathbb{E}[A_{\widehat{\tau}_{stab}}].
	\end{align*}
	\label{lem:T_stab}
\end{lemma}
\begin{proof}

	The first inequality follows as the time taken to reach $T_{stab}$ is upper bounded by the time it takes all agents to reach phase $\widehat{\tau}_{stab}$ after time $\Gamma$. However, by definition we know that all agents last in any phase $j$ after time $\Gamma$ for at-most $(1+\delta)(A_{j+1}-A_j)$ arm-pulls. The upper bound is concluded by noticing that an agent can be in a phase no smaller than $0$ at time $\Gamma$ and subsequently it takes an agent a maximum of $(1+\delta)A_{\widehat{\tau}_{stab}}$ time to reach phase $\widehat{\tau}_{stab}$. 
\end{proof}

\begin{lemma}
	Almost-surely, we have
	\begin{align*}
	H \leq \Gamma + \lceil 1+\delta \rceil \widehat{\tau}_{stab} .
	\end{align*}
	\label{lem:H_calc}
\end{lemma}
\begin{proof}
Notice that at time $\Gamma$, the maximum phase any agent can be in is $\Gamma$. This follows from the trivial upper bound, where in each time step, an agent increases its phase by one in each time slot. After time $\Gamma$, we know by definition, that any agent plays arms at-least $A_j - A_{j-1}$ times and at-most $(1+\delta)(A_j - A_{j-1})$ in phase $j$. Thus, the total number of phase changes an agent will have in the time interval $[\Gamma+1,T_{stab}]$ is at-most $A^{-1}(T_{stab}-\Gamma)$. Thus, we get
\begin{align*}
    H &\leq \Gamma + A^{-1}(T_{stab} - \Gamma), \\
    &\stackrel{(a)}{\leq} \Gamma + A^{-1}((1+\delta)A_{\widehat{\tau}_{stab}}),\\
    &\stackrel{(b)}{\leq} \Gamma + A^{-1}(A_{\lceil 1+\delta \rceil \widehat{\tau}_{stab}}), \\
    &\leq \Gamma + \lceil 1+\delta \rceil \widehat{\tau}_{stab}.
\end{align*}
Step $(a)$ follows from Lemma \ref{lem:T_stab}, step $(b)$ follows from convexity of $(A_x)_{x \geq 1}$ and the last inequality follows from the definition of $A^{-1}(\cdot)$.

\end{proof}

\subsection{Quantitative Results}
In this section, we compute quantitative bounds in terms of the algorithm' input parameters. 

\begin{prop}
	For all $x \geq 2$ and $\delta > 0$, 
	\begin{align*}
	\mathbb{P}[H^{*} > l ] &\leq 2N\sum_{x \geq l}e^{-c(\delta)(A_{x} - A_{x-1})},
	\end{align*}
	where $c(\delta) = \min  \left( \frac{\delta}{2} + \ln \left( 1 + \frac{\delta}{2} \right),  (1+\delta)\ln\left(\frac{2+2\delta}{2+\delta}\right)  - \frac{\delta}{2} \right)$.
	\label{prop:H_gamma_tail}
\end{prop}
\begin{proof}
	From the definition of $H^{*}$, we have 
	\begin{align*}
	\mathbb{P}[H^{*} &\geq l] = \mathbb{P} \left[ \bigcup_{i=1}^{N} \bigcup_{x \geq l} \text{Poisson}\left(\left(1+\frac{\delta}{2} \right) (A_{x} - A_{x-1})  \right) \notin \left[(A_{x} - A_{x-1}), (1+\delta)(A_{x} - A_{x-1}) \right]\right],\\
	&\leq N \sum_{x \geq l}\mathbb{P} \left[ \text{Poisson}\left(\left(1+\frac{\delta}{2} \right) (A_{x} - A_{x-1})  \right) \notin \left[(A_{x} - A_{x-1}), (1+\delta)(A_{x} - A_{x-1}) \right]\right], \\
	&\leq N \sum_{x \geq l} (2e^{-2c(\delta)(A_{x} - A_{x-1})} ).
	\end{align*}
	In the last inequality, we use the classical large-deviation estimate for a Poisson random variable.
\end{proof}

\begin{lemma}
	\begin{align*}
	\mathbb{E}[\Gamma] \leq 2N \left( A_{x_0}^2 e^{-c(\delta)(A_{x_0}-A_{x_0-1})}+\sum_{x \geq 1} A_x  e^{-c(\delta)A_{x-1}^{\omega}} \right) + \frac{N}{c(\delta)},
	\end{align*}
	where $c(\delta)$ is given in Proposition \ref{prop:H_gamma_tail} and $x_0$ is from Assumption {\bf A.2} in Section \ref{sec:assumptions}.
	\label{lem:gamma_calc}
\end{lemma}
\begin{proof}
	We start by computing the tail probability $\mathbb{P}[\Gamma > t]$. The key observation to do so is the following inequality. For every $L \geq 0$, we have
	\begin{align*}
	\mathbb{P}[\Gamma \geq t] &\leq \mathbb{P}[H^{*} \geq L ] + \mathbb{P} \left[ \bigcup_{i=1}^{N} \sum_{j=0}^{L} \mathcal{P}_j^{(i)} \leq t\right].
	\end{align*}
	We will then compute $\mathbb{E}[\Gamma]$ by choosing $L = A^{-1}(t)$. We shall compute each of these terms separately.
	\begin{align*}
	\mathbb{P}[H^{*} \geq A^{-}(t)] &= \mathbb{P}\left[\bigcup_{i=1}^{N} \bigcup_{x\geq A^{-1}(t)} \mathcal{P}_x^{(i)} \not \in [(A_{x} - A_{x-1}), (1+\delta)(A_{x} - A_{x-1})] \right], \\
	&\leq N \sum_{x \geq A^{-1}(t)} 2e^{-c(\delta)(A_{x} - A_{x-1})},
	\end{align*}
	where the second inequality follows from Proposition \ref{prop:H_gamma_tail}. Similarly, standard large deviation estimates for Poisson random variables (observe that for all $L$,$\sum_{j=0}^{L} \mathcal{P}_j^{(i)}$ is Poisson distributed with mean $L$ ) and union bound gives
	\begin{align*}
	 \mathbb{P} \left[ \bigcup_{i=1}^{N} \sum_{j=0}^{L} \mathcal{P}_j^{(i)} \leq t\right] \leq N e^{-c(\delta)t}.
	\end{align*}
	Thus, we can bound $\mathbb{E}[\Gamma]$ as 
	\begin{align*}
	\mathbb{E}[\Gamma] &\leq \sum_{t \geq 1} \mathbb{P}[\Gamma \geq t], \\
	&\leq 2N\sum_{t \geq 1} \sum_{x \geq A^{-1}(t)}e^{-c(\delta)(A_{x} - A_{x-1})} + N\sum_{t \geq 1}e^{-c(\delta)t},\\
	&\stackrel{(a)}{\leq} 2N \sum_{x \geq 1} \sum_{t=1}^{A_x} e^{-c(\delta)(A_{x} - A_{x-1})} + N\sum_{t \geq 1}e^{-c(\delta)t},\\
	&\stackrel{(b)}{\leq } 2N \left( A_{x_0}^2 e^{-c(\delta)(A_{x_0}-A_{x_0-1})}+\sum_{x \geq 1} A_x  e^{-c(\delta)A_{x-1}^{\omega}} \right) + N\int_{t\geq 0} e^{-c(\delta)t}dt, \\
	& = 2N \left( A_{x_0}^2 e^{-c(\delta)(A_{x_0}-A_{x_0-1})}+\sum_{x \geq 1} A_x  e^{-c(\delta)A_{x-1}^{\omega}} \right) + \frac{N}{c(\delta)}.
	\end{align*}
	Step $(a)$ follows from changing the order of summation (which is licit as all terms are positive) and step $(b)$ follows from the assumption {\bf A.2} in Section \ref{sec:assumptions}. Standard results from analysis gives that the series in the last display is finite as $c(\delta)> 0$ and $A_x \leq A_{2x} \leq A_{x-1}^{3}$, where the second inequality follows from Assumption {\bf A.2} in Section \ref{sec:assumptions}. 
\end{proof}

\begin{lemma}
For all $\delta > 0$ such that $c(\delta) > \frac{5}{4}D$ and $(3+2\delta + \ln (4+2\delta)) \geq \frac{5}{4}D^{-1}$, where $c(\delta)$ is given in Proposition \ref{prop:H_gamma_tail} and $D$ is in Assumption {\bf{A.2}} in Section \ref{sec:assumptions},
\begin{align*}
    \mathbb{E}[A_{4\Gamma}] \leq  2N\sum_{x \geq 1}A_{ A_x }e^{-c(\delta)(A_{x} - A_{x-1})} + N \sum_{t \geq 1}e^{ -(3+2\delta + \ln (4+2\delta))A^{-1}(t) }.
\end{align*}
    \label{lem:A_Gamma}
\end{lemma}
\begin{proof}
Observe that $\mathbb{E}[A_{4\Gamma}] \leq \sum_{t \geq 1} \mathbb{P}\left[ \Gamma \geq \frac{1}{4} A^{-1}(t) \right]$. We use similar ideas as in Lemma \ref{lem:gamma_calc} to bound the tail probability. Recall that for any $t \geq 1$ and any $L \geq 1$, the following bound holds
    \begin{align*}
        	\mathbb{P}[\Gamma \geq t] &\leq \mathbb{P}[H^{*} \geq L ] + \mathbb{P} \left[ \bigcup_{i=1}^{N} \sum_{j=0}^{L} \mathcal{P}_j^{(i)} \leq t\right],\\
        	&\leq \mathbb{P}[H^{*} \geq L ] + N\mathbb{P} \left[  \text{Poisson} \left( \left( 1 + \frac{\delta}{2}\right)A_L\right) \leq t\right].
    \end{align*}
    In this proof, we shall use $L = A^{-1} \left( A^{-1}(t) \right)$. Thus,
    \begin{align*}
    \mathbb{P}\left[ \Gamma \geq \frac{1}{4} A^{-1}(t) \right] &\leq \mathbb{P}\left[H^{*} \geq A^{-1} \left( A^{-1}(t) \right) \right] + N\mathbb{P} \left[  \text{Poisson} \left( \left( 1 + \frac{\delta}{2}\right)A_{A^{-1} \left( A^{-1}(t) \right)}\right) \leq \frac{1}{4} A^{-1}(t) \right],\\
    &= \mathbb{P}\left[H^{*} \geq A^{-1} \left( A^{-1}(t) \right) \right] + N\mathbb{P} \left[  \text{Poisson} \left( \left( 1 + \frac{\delta}{2}\right)  A^{-1}(t) \right) \leq \frac{1}{4} A^{-1}(t) \right],\\
    &\leq 
    2N \sum_{x \geq A^{-1} \left(\frac{1}{M} A^{-1}(t) \right)} e^{-c(\delta)(A_{x} - A_{x-1})}
    + N e^{ -(3+2\delta + \ln (4+2\delta))A^{-1}(t) }.
    \end{align*}
    The last display follows from Proposition \ref{prop:H_gamma_tail} and standard Poisson random variable Chernoff bound. Thus, we can bound $\mathbb{E}[A_{4\Gamma}]$ as 
    \begin{align*}
    \mathbb{E}[A_{4\Gamma}] &\leq \sum_{t \geq 1}  \mathbb{P}\left[ \Gamma \geq \frac{1}{4} A^{-1}(t) \right], \\
    &\leq 2N \sum_{t \geq 1}\sum_{x \geq A^{-1} \left( A^{-1}(t) \right)}e^{-c(\delta)(A_{x} - A_{x-1})} + N \sum_{t \geq 1}e^{ -(3+2\delta + \ln (4+2\delta))A^{-1}(t) }, \\
    &\stackrel{(a)}{=} 2N\sum_{x \geq 1} \sum_{t=1}^{A\left( A_x \right)}e^{-c(\delta)(A_{x} - A_{x-1})} + N \sum_{t \geq 1}e^{ -(3+2\delta + \ln (4+2\delta))A^{-1}(t) }, \\
    &= 2N\sum_{x \geq 1}A_{A_x }e^{-c(\delta)(A_{x} - A_{x-1})} + N \sum_{t \geq 1}e^{ -(3+2\delta + \ln (4+2\delta))A^{-1}(t) }.
    \end{align*}
    We will choose $\delta$ sufficiently large so that both the series are convergent. This is possible as the maps $\delta \longrightarrow c(\delta)$ and $\delta \longrightarrow (3+2\delta + \ln (4+2\delta))$ are non-decreasing and $\lim_{\delta \rightarrow \infty}c(\delta) = \lim_{\delta \rightarrow \infty} (3+2\delta + \ln (4+2\delta)) = \infty$. Observe that since $A_x \leq e^{Dx}$, for all large $x$, we have $A^{-1}(t) \geq \frac{1}{D}\ln(t)$. Thus, if $c(\delta) \geq \frac{5}{4}D$ and $(3+2\delta + \ln (4+2\delta)) > D^{-1}$, both series are convergent.
\end{proof}

\begin{lemma}
	For any $C \geq 2$, 
	\begin{align*}
	\mathbb{E}[A_{C \widehat{\tau}_{stab}}] \leq A_{\lceil \frac{C}{2} \rceil j^{*}} + \left(\frac{2}{2\alpha-3}\right)\sum_{l \geq3} \frac{A_{2l}}{A_{l-1}^{3}} + 2\sum_{x \geq  \lceil \frac{j^{*}}{2} \rceil} (A_{\lceil Cx \rceil})^{\left(2(2\alpha-6)+2\right)}e^{-c(\delta) (A_{x} - A_{x-1})},
	\end{align*}
	where $c(\delta)$ is given in Proposition \ref{prop:H_gamma_tail} and $j^{*}$ is given in Theorem \ref{thm:strong_result}.
	\label{lem:A_tau_prob}
\end{lemma}
\begin{proof}
	We start with the definition of expectation and repeatedly applying union bound yields,
	\begin{align*}
		\mathbb{E}[A_{C \widehat{\tau}_{stab}}]  =& \sum_{t \geq 1} \mathbb{P}[A_{C \widehat{\tau}_{stab}} \geq t ], \\
		&\leq \sum_{t \geq 1} \mathbb{P}\left[\widehat{\tau}_{stab} \geq \frac{1}{C} A^{-1}(t)\right], \\
		&\leq A_{\lceil \frac{C}{2} \rceil j^{*}} + \sum_{t \geq A_{\lceil \frac{C}{2} \rceil j^{*}} +1} \mathbb{P}\left[\widehat{\tau}_{stab} \geq \frac{1}{C} A^{-1}(t)\right],\\
		&\leq  A_{\lceil \frac{C}{2} \rceil j^{*}} + \sum_{t \geq A_{\lceil \frac{C}{2} \rceil j^{*}} +1} \mathbb{P} \left[ \bigcup_{i=1}^{N} \bigcup_{l \geq \frac{1}{C} A^{-1}(t)} \chi_l^{(i)} = 0 \right],\\
		&\leq  A_{\lceil \frac{C}{2} \rceil j^{*}} + \sum_{t \geq A_{\lceil \frac{C}{2} \rceil j^{*}} +1} N  \mathbb{P} \left[  \bigcup_{l \geq \frac{1}{C} A^{-1}(t)} \chi_l^{(i)} = 0 \right],\\
		&\leq  A_{\lceil \frac{C}{2} \rceil j^{*}} + \sum_{t \geq A_{\lceil \frac{C}{2} \rceil j^{*}} +1} N \sum_{l \geq \frac{1}{C} A^{-1}(t)} \left( \mathbb{P} [ \chi_l^{(i)} = 0, l \geq H^{*}] + \mathbb{P} [ \chi_l^{(i)} = 0, l < H^{*}] \right), \\
		&\leq  A_{\lceil \frac{C}{2} \rceil j^{*}} + \sum_{t \geq A_{\lceil \frac{C}{2} \rceil j^{*}} +1} N \sum_{l \geq \frac{1}{C} A^{-1}(t)} \left( \mathbb{P} [ \chi_l^{(i)} = 0, l \geq H^{*} ] + \mathbb{P} [ l < H^{*}] \right), \\
		&= A_{\lceil \frac{C}{2} \rceil j^{*}} + \sum_{t \geq A_{\lceil \frac{C}{2} \rceil j^{*}} +1} N \sum_{l \geq \frac{1}{C} A^{-1}(t)}\mathbb{P} [ \chi_l^{(i)} = 0, l \geq H^{*}] +  \sum_{t \geq A_{\lceil \frac{C}{2} \rceil j^{*}} +1} \sum_{l \geq \frac{1}{C} A^{-1}(t)} N  \mathbb{P} [ l < H^{*} ], \\
		&\stackrel{(a)}{\leq} A_{\lceil \frac{C}{2} \rceil j^{*}} + \left(\frac{2}{2\alpha-3}\right)\sum_{l \geq3} \frac{A_{2l}}{A_{l-1}^{3}} + \sum_{t \geq A_{\lceil \frac{C}{2} \rceil j^{*}} +1} \sum_{l \geq \frac{1}{C} A^{-1}(t)} N  \mathbb{P} [ l < H^{*}].
	\end{align*}
	Step $(a)$ follows as $C \geq 2$, and hence, the first summation follows from identical calculations as carried out in Proposition \ref{prop:strong_cost}. This is so as the bound in Lemma \ref{lem:error_estimate_ucb} and in Lemma \ref{lem:error_estimate_ucb_random} are identical. Thus, the first series is upper bounded by $\left(\frac{2}{2\alpha-3}\right)\sum_{l \geq 3} \frac{A_{2l}}{A_{l-1}^{3}}$. We shall now estimate the second series. 
	\begin{align*}
	 \sum_{t \geq A_{\lceil \frac{C}{2} \rceil j^{*}} +1} \sum_{l \geq \frac{1}{C} A^{-1}(t)} N  \mathbb{P} [ l < H^{*} ] &\leq \sum_{l \geq \lceil \frac{j^{*}}{2} \rceil} \sum_{t =A_{\lceil \frac{C}{2} \rceil j^{*}} +1}^{A_{\lceil Cl \rceil }}  N  \mathbb{P} [ l < H^{*} ],\\
	 &\leq  \sum_{l \geq \lceil \frac{j^{*}}{2} \rceil} N A_{\lceil Cl \rceil }  \mathbb{P} [ l < H^{*} ],\\
	 &\stackrel{(b)}{\leq}  \sum_{l \geq \lceil \frac{j^{*}}{2} \rceil}  N^2 A_{\lceil Cl \rceil }\sum_{x \geq l}2e^{-c(\delta)  (A_{x} - A_{x-1})}, \\
	 &=N^2 \sum_{x \geq  \lceil \frac{j^{*}}{2} \rceil}\sum_{l= \lceil \frac{j^{*}}{2} \rceil}^{x} A_{\lceil Cl \rceil} 2e^{-c(\delta) (A_{x} - A_{x-1})},\\
	 &\leq N^2 \sum_{x \geq  \lceil \frac{j^{*}}{2} \rceil} x A_{\lceil Cx \rceil}e^{-c(\delta) A_{x-1}^{\omega}}, \\
	 &\stackrel{(c)}{\leq} \sum_{x \geq  \lceil \frac{j^{*}}{2} \rceil} (A_{\lceil Cx \rceil})^{2(2\alpha-6)+2}2e^{-c(\delta)  (A_{x} - A_{x-1})}, \\
	 &\stackrel{(d)}{<} \infty.
	\end{align*}
	
Step $(b)$ follows from Proposition \ref{prop:H_gamma_tail} and in step $(c)$, we use the fact that for all $x \geq \frac{j^{*}}{2}$, we have $N \leq A_x^{2 \alpha-6}$. Step $(d)$ follows  from Assumption {\bf A.2} that for all sufficiently large $l$, $A_{2l} \leq A_l^3$, which on iterating yields that for all large $x$ and any $C \geq 2$, we have $A_{Cx} \leq A_x^{3^{\lceil \log_2(C) \rceil}}$. Thus, we have the following chain of inequalities.
%
%
%
\begin{align*}
\sum_{x \geq  \lceil \frac{j^{*}}{2} \rceil} (A_{\lceil Cx \rceil})^{2(2\alpha-6)+2}2e^{-c(\delta) (A_{x} - A_{x-1})}	 &\stackrel{}{\leq}  \sum_{x \geq  \lceil \frac{j^{*}}{2} \rceil} (A_{x})^{3^{\lceil \log_2(C) \rceil}\left(2(2\alpha-6)+2\right)}2e^{-c(\delta) (A_{x} - A_{x-1})}, \\
 &\stackrel{(e)}{\leq} D_1\sum_{x \geq 2} A_x^{D_2} e^{-cA_{x-1}^{\omega}} < \infty.
\end{align*}
for some $D_1,D_2, \omega^{'} > 0$. Step $(e)$ follows as we can replace the tail terms of the series with  $A_{x} - A_{x-1} \geq A_{x-1}^{\omega}$ from Assumption {\bf A.2}. The finiteness of the series in $(e)$  is a standard fact from real analysis and can be proven for instance through a Taylor series approximation of the exponential function.
\end{proof}

\subsection{Proof of Theorem \ref{thm:poisson_result}}

The proof of Theorem \ref{thm:prob_main_result} is concluded by using estimates in Lemmas \ref{lem:gamma_calc} and \ref{lem:A_tau_prob} into Lemmas \ref{lem:T_prob} and \ref{lem:prob_reg_decompose}. 

\section{Proof of Theorem \ref{thm:prob_main_result}}
\label{sec:prob_algo_proof}

The proof follows identical steps as that of Theorem \ref{thm:poisson_result}, with the exception that $H^{*} = \Gamma = 0$ almost-surely. More precisely, substituting these two facts in Lemma \ref{lem:T_prob} and re-using all the remaining structural and quantitative results from the proof of Theorem \ref{thm:poisson_result} will yield the desired result.

\section{Auxillary Results}
\begin{prop}
	For each $\delta >0$, $y \geq 0$ and convex sequence $(A_j)_{j \geq 0}$, we have
	\begin{align*}
	\sup\{ j+x \geq 0 : \exists j \leq y+1, (A_{j+x}-A_j) \leq (1+\delta)(A_{y+1} - A_y)\} \leq y + \lfloor2+\delta\rfloor
	\end{align*}
	\label{prop:convex}
\end{prop}
\begin{proof}
	Let $x \geq 0$ and $j \leq y+1$ be such that
	\begin{align}
	(A_{j+x} + (1+\delta)A_{y}) \leq (1+\delta)A_{y+1} + A_j.
	\label{eqn:convex1}
	\end{align}
	Now since $j \leq y+1$, and $(A_l)_{l \geq 0}$ is non-decreasing, the above inequality implies
	\begin{align*}
	\frac{A_{j+x} + (1+\delta)A_y}{2+\delta} \leq A_{y+1}.
	\end{align*}
	Now, let $j+x = y+ \lfloor2+\delta\rfloor + k$, for some $k \geq 0$. From convexity of $(A_j)_{j \geq 1}$, we have 
	\begin{align*}
	A_{y + \frac{\lfloor2+\delta\rfloor + k}{2+\delta}} \leq  \frac{A_{j+x} + (1+\delta)A_y}{2+\delta} \leq A_{y+1}
	\end{align*}
	But for all $k \geq 1$, we have $A_{y + \frac{\lfloor2+\delta\rfloor + k}{2+\delta}} > A_{y+1}$ and hence $k=0$ is the only possibility such that Equation (\ref{eqn:convex1}) holds.
\end{proof}

\section{Proof of Theorem \ref{thm:full_interaction}}
\label{sec:lb_proof}

\begin{proof}

In order to prove the  bound, we shall consider a system of \emph{full interaction among agents}, where there are no constraints on communications. In this system, each agent after pulling an arm and observing a reward, communicates this information (the arm pulled and reward observed) to central \emph{board}. Thus, at the beginning of each time-step, every agent has access to the entire system history (arms pulled and rewards obtained) up-to the previous time step, by which to base the current time step's action (arm pull) on. As all agents have access to the same history at the beginning of a time step, the optimal strategy to minimize per agent regret is one where in each time step, all agents play the same arm. Hence, this system is equivalent to a single \emph{leader} playing arms, such that on playing any arm at any time, the leader observes $N$ i.i.d. reward samples from the chosen arm, each corresponding to the obtained reward by the agents. From henceforth, we mean by the full interaction setting, as one wherein a single leader agent pull an arm at each time step, and observes $N$ i.i.d. reward samples from the chosen arm. 
\\

By construction, a lower bound for regret incurred by the leader agent in the full interaction setting forms a lower bound on the per-agent regret in our model with communication constraints. This is so, since the leader agent in full interaction setting can `simulate' any feasible policy of any agent $i \in [N]$ with communication constraints among agents. Notice that each time the leader agent in the full-interaction setting plays an arm, it receives $N$ i.i.d. samples of rewards, corresponding to the reward on that arm obtained by the $N$ agents. We will consider an alternate system where a \emph{fictitious leader agent} plays for $NT$ time steps, where at each time, the fictitious agent is playing arms, as a measurable function of its observed history. From standard results, (for eg. \citep{lai_robbins}), the total regret of the fictitious agent, after $NT$ arm-pulls satisfies  
        \begin{align}
        \liminf_{T \rightarrow \infty} \frac{\mathbb{E}[R_{NT}^{(\text{fictitious})}]}{\ln(NT)} \geq \left( \sum_{j = 1}^{K-1} \frac{\Delta_j}{\text{KL}(\mu_j,\mu_1)} \right),
        \label{eqn:ficticious_agent}
    \end{align}
    
    
    
   Now, we shall argue that the preceding display implies the desired lower bound on per-agent regret in the full interaction setting. Fix some $a \in \{0,\cdots,N-1\}$. Denote, by the regret incurred by the fictitious agent at time steps $a,N+a,\cdots,N(T-1)+a$ as $\mathcal{R}_a^{(f)}$. Clearly $\sum_{a=1}^{N} \mathcal{R}_a^{(f)} = \mathbb{E}[R_{NT}^{(\text{fictitious})}]$. 
   \\
   
   Denote by $\boldsymbol{\Pi}_{agent}$ to be the set of consistent policies for the agents in the full-interaction setting and by  $\boldsymbol{\Pi}_{fictitious}$ as the set of all consistent policies for the fictitious agent. Denote by the set of policies $\widetilde{\boldsymbol{\Pi}}_{fictitious} \subset \boldsymbol{\Pi}_{fictitious}$, as those policies for the fictitious agents, where for any policy $\pi \in\widetilde{\boldsymbol{\Pi}}_{fictitious}$, the arms played at time instants $N,2N,\cdots,NT$, belong to $\boldsymbol{\Pi}_{agent}$. Furthermore, for all $a \in \{1,\cdots,T\}$, and all $b \in \{1,\cdots,N-1\}$, and all $\pi \in \widetilde{\boldsymbol{\Pi}}_{fictitious}$, the arm chosen by $\pi$ at time instant $aN$ is the same as the arm chosen at time-instant $aN+b$. In other words, the the set of policies $ \widetilde{\boldsymbol{\Pi}}_{fictitious}$ are the ones that any agent under the full interaction setting of our model can play. This definitions now give us for any $a \in \{0,\cdots,N-1\}$

   \begin{align*}
     \inf_{\pi \in \boldsymbol{\Pi}_{agent}}\mathbb{E}[R_T^{(i)}]& = \inf_{\pi \in \widetilde{\boldsymbol{\Pi}}_{fictitious}} \mathcal{R}_a^{(f)}, \\
     &=\inf_{\pi \in \widetilde{\boldsymbol{\Pi}}_{fictitious}}\frac{1}{N}\sum_{a=1}^{N} \mathcal{R}_a^{(f)},\\
     &= \inf_{\pi \in \widetilde{\boldsymbol{\Pi}}_{fictitious}} \frac{1}{N} \mathbb{E}[R_{NT}^{(\text{fictitious})}], \\
     &\geq \inf_{\pi \in \boldsymbol{\Pi}_{fictitious}} \frac{1}{N} \mathbb{E}[R_{NT}^{(\text{fictitious})}].
    \end{align*} 
    The first equality follows as under any policy in $\widetilde{\boldsymbol{\Pi}}_{fictitious}$, the arms played by the fictitious agent only chooses potentially new arms to play at instants $N,2N,\cdots$. Now, using Equation (\ref{eqn:ficticious_agent}), we get from the previous display, that for any policy $\pi \in \boldsymbol{\Pi}_{agent}$,
    
    \begin{equation*}
    \liminf_{T \rightarrow \infty} \frac{\mathbb{E}[R_{T}^{(i)}]}{\ln(NT)} \geq   \left( \frac{1}{N}\sum_{j \geq 1} \frac{\Delta_j}{\text{KL}(\mu_j,\mu_1)} \right).
\end{equation*}
\end{proof}

\section{Proof of Corollary \ref{cor:main_result_interpretation}}
\label{appendix-proof-interpretation}
In order to prove the corollary, we first establish that $A_x \leq 2x^{\beta}$, for all small $\varepsilon$ in Equation (\ref{eqn:A_x_defn}). Notice from Equation (\ref{eqn:A_x_defn}) that for all $x \in \mathbb{N}$, we have
\begin{align*}
    A_x &= \max \left( \min \{t \in \mathbb{N} : B_t \geq x \}, \lceil (1+x)^{1+\varepsilon} \rceil \right), \\
    &=\max\left( \min \{t \in \mathbb{N} : B_t \geq x \}, \lceil (1+x)^{1+\varepsilon} \rceil \right),\\
    &\leq \max \left( x^{\beta}, (1+x)^{1+\varepsilon} \right),\\
    & \leq \max(2x^{\beta},2x^{1+\varepsilon}), \\
    &=2x^{\beta},
\end{align*}
where the last equality follows since $\varepsilon < \beta - 1$. 
Furthermore, for all $x\geq x_0$ where $\varepsilon < \beta \frac{\ln(x_0)}{\ln(x_0+1)}-1$, we have $A_x = x^{\beta}$. Such a $x_0$ exists since $\beta - 1 > 0$. Moreover, from definition of $A_x$, we have $A_x \geq x^{\beta}$, for all $x$.
\\

Recall that $g((A_x)_{x \in \mathbb{N}}) = A_{j^*}+\frac{2}{2\alpha-3}\sum_{l \geq \frac{j^*}{2}-1} \frac{A_{2l+1}}{A_{l-1}^3}$.
We first bound the series term in  as follows
\begin{align}
    \sum_{l \geq \frac{j^*}{2}-1} \frac{A_{2l+1}}{A_{l-1}^3} &\stackrel{(a)}{\leq} \sum_{l \geq 2} 2\frac{(2l+1)^{\beta}}{(l-1)^{3 \beta}},\\
    &\leq 2  \sum_{l \geq 2}3^{\beta} \frac{1}{(l-1)^{2\beta}},\nonumber \\
    &\leq 2  \frac{\pi^2}{6} 3^{\beta}. \label{eqn:proof-interp-series}
\end{align}

We now bound $j^*$ in this case. Recall that 
\begin{align*}
    j^{*}&=2 \max \bigg( A^{-1} \left( \left(N{K \choose 2} \left(\bigg \lceil \frac{K}{N} \bigg\rceil + 1\right) \right)^{\frac{1}{(2\alpha-6)}} \right) +1 , \min \left\{ j \in \mathbb{N} : \frac{A_j - A_{j-1}}{2 + \lceil \frac{K}{N}\rceil } \geq 1 + \frac{4 \alpha \log(A_j)}{\Delta_{2}^{2}} \right\}\bigg),\\
    &\leq 2 \max \left( K^{\frac{3}{\beta (2\alpha-6)}}+1, \min \left\{ j \in \mathbb{N} : \frac{j^{\beta} - (2(j-1))^{\beta}}{2 + \lceil \frac{K}{N}\rceil } \geq 1 + \frac{4 \alpha \log(j^{\beta})}{\Delta_{2}^{2}} \right\}\right),\\
    &\leq  2 \max \left( K^{\frac{3}{\beta (2\alpha-6)}}, \min \left\{ j \in \mathbb{N} : \frac{j^{\beta} - (2(j-1))^{\beta}}{2 + \lceil \frac{K}{N}\rceil } \geq  \frac{8 \alpha \log(j^{\beta})}{\Delta_{2}^{2}} \right\} \right),\\
    &\leq 2\max \left( K^{\frac{3}{\beta (2\alpha-6)}},\left( 16\alpha\frac{2 + \lceil  \frac{K}{N}\rceil}{\Delta_2^2} \right)^{\frac{1}{\beta-1}} \right).
\end{align*}
Thus, we have
\begin{align}
    A_{j^*} &\leq 2 (j^*)^{\beta},\nonumber \\
    &\leq 4 \max \left( K^{\frac{3 }{ (2\alpha-6)}},\left( 16\alpha\frac{2 + \lceil  \frac{K}{N}\rceil}{\Delta_2^2} \right)^{\frac{\beta}{\beta-1}} \right). \label{eqn:proof-interp-j-start}
\end{align}. 
Thus from Equations (\ref{eqn:proof-interp-series}) and (\ref{eqn:proof-interp-j-start}), we get that 
\begin{align*}
    g((A_x)_{x \in \mathbb{N}}) \leq \frac{4}{2\alpha-3} \frac{\pi^2}{6} 3^{\beta} + 4 \max \left( K^{\frac{3 }{ (2\alpha-6)}},\left(16\alpha \frac{2 + \lceil  \frac{K}{N}\rceil}{\Delta_2^2} \right)^{\frac{\beta}{\beta-1}} \right).
\end{align*}
The proof is completed thanks to the formula in Corollary \ref{cor:regret-comm}.
\section{Impact of Gossip Matrix $P$}
\label{appendix_proof_cor_reg_comm_tradeoff}

\begin{corollary}
Suppose $N \geq 2$ agents are connected by a $d$-regular graph with adjacency matrix $\boldsymbol{A}_G$ having conductance $\phi$ and the gossip matrix $P = d^{-1}\boldsymbol{A}_G$. If the agents are using Algorithm \ref{algo:main-algo}
with parameters satisfying assumptions in Theorem \ref{thm:strong_result}, then for any $i \in [N]$ and $T \in \mathbb{N}$ 
\begin{equation*}
\mathbb{E}[R_T^{(i)}] \leq \underbrace{{4 \alpha \ln(T)}\left( \sum_{j=2}^{\lceil \frac{K}{N} \rceil + 2} \frac{1}{\Delta_j} \right) + \frac{K}{4}}_{\text{\clap{Collaborative UCB Regret}}} +     \underbrace{A_{2C\frac{ \log(N)}{\phi}} + g\left( (A_x)_{x \in \mathbb{N}} \right) + A_{j^{*}}+1}_{\text{\clap{Cost of Pairwise Communications}}},
\end{equation*}
where $g(\cdot)$ is  from Theorem \ref{thm:strong_result}, and $C > 0$ is an universal constant stated in Lemma \ref{lem:rumor_spreading_lattanzi} in the Appendix. Similarly, if all agents run Algorithm \ref{algo:main-algo-prob} with
assumptions as in Theorem \ref{thm:prob_main_result}, then 
\begin{equation*}
    \mathbb{E}[R_T^{(i)}] \leq \underbrace{{4 \alpha \ln(T)}\left( \sum_{j=2}^{\lceil \frac{K}{N} \rceil + 2} \frac{1}{\Delta_j} \right) + \frac{K}{4}}_{\text{{Collaborative UCB Regret}}} +   \underbrace{(1+\delta)A_{2\lfloor 2 + \delta \rfloor C\frac{ \log(N)}{\phi}} + \widehat{g}\left( (A_x)_{x \in \mathbb{N}},\delta \right) +1}_{\text{{Cost of Pairwise Communications}}},
\end{equation*}
where $\widehat{g}(\cdot)$ is given in Theorem \ref{thm:prob_main_result}.
\label{cor:regret-comm}
\end{corollary}

\begin{proof}
The proof follows if we establish that $\mathbb{E}[A_{2\tau_{spr}^{(P)}}] \leq A_{\frac{2C \log(N)}{\phi}}+1$ and $\mathbb{E}[A_{2\lfloor 2 + \delta \rfloor \tau_{spr}^{(P)}}] \leq  A_{\frac{2\lfloor 2 + \delta \rfloor C \log(N)}{\phi}}+1$. We can bound them using the main result from \citep{rumor_lattanzi}, restated as Lemma \ref{lem:rumor_spreading_lattanzi} in the sequel. That lemma in particular gives that, one can compute $\mathbb{E}[A_{2\tau_{spr}^{(P)}}]$ as follows.
    \begin{align*}
        \mathbb{E}[A_{2 \tau_{spr}^{(P)}}] &\leq A_{ \frac{2C \log(N)}{\phi}} + \sum_{t \geq A_{ \frac{2C \log(N)}{\phi}}} \mathbb{P}[A_{2 \tau_{spr}^{(P)}} \geq t],\\
        &\leq A_{ \frac{2C \log(N)}{\phi}} + \sum_{l \geq 1} \mathbb{P}\left[A_{ 2 \tau_{spr}^{(P)} }\geq A_{\frac{2Cl \log(N)}{\phi}}\right] A_{\frac{2Cl \log(N)}{\phi}}, \\
        &\leq A_{ \frac{2C \log(N)}{\phi}} + \sum_{l \geq 1} \mathbb{P}\left[{ 2 \tau_{spr}^{(P)} }\geq {\frac{2Cl \log(N)}{\phi}}\right] A_{\frac{2Cl \log(N)}{\phi}}, \\
        &\stackrel{(a)}{\leq} A_{ \frac{2C \log(N)}{\phi}} + \sum_{l \geq 1} e^{-4l {\log(N)}}A_{\frac{2Cl \log(N)}{\phi}}, \\
        &\stackrel{(b)}{\leq} A_{ \frac{2C \log(N)}{\phi}} + \sum_{l \geq 1} e^{-2l{\log(N)}}, \\
        &\stackrel{(c)}{\leq}A_{ \frac{2C \log(N)}{\phi}} +1.
    \end{align*}
    
    In step $(a)$, we use the estimate from Lemma \ref{lem:rumor_spreading_lattanzi}. In step $(b)$, we use 
    the additional assumption in the corollary that $A_l \leq e^{Dl}$, for all $D > 0$. Thus, we can choose $D \leq \frac{\phi}{C}$ to arrive at the conclusion in step $(b)$. In step $(c)$, we use $N \geq 2$ to bound the geometric series. Similar computation will yield the bound on $\mathbb{E}[A_{2\lfloor 2 + \delta \rfloor \tau_{spr}^{(P)}}]$.
\end{proof}

\begin{lemma}
There exists an universal constant $C > 0$, such that for every $d \geq 2$ regular graph on $N$ vertices with conductance $\phi$, the spreading time of the standard PULL process completes in time $\tau_{spr}^{(P)}$ which satisfies for all $l \in \mathbb{N}$,
\begin{align*}
    \mathbb{P}\left[\tau_{spr}^{(P)} \geq C l \frac{\log(N)}{\phi}\right] \leq N^{-4l}.
\end{align*}
\label{lem:rumor_spreading_lattanzi}
\end{lemma}
\begin{proof}
    The main result (Lemma $6$) of \citep{rumor_lattanzi} gives that there exists a constant $C > 0$, such that for all $d$-regular graphs with conductance $\phi$, the spreading time satisfies 
    \begin{align*}
    \mathbb{P}\left[\tau_{spr}^{(P)} \geq C  \frac{\log(N)}{\phi}\right] \leq N^{-4}.
    \end{align*}
    Now, given any $l \in \mathbb{N}$, we can now divide the time into intervals $\left[0, C \frac{\log(N)}{\phi} \right], \left[C \frac{\log(N)}{\phi}, 2C \frac{\log(N)}{\phi} \right],\cdots \\, \left[C (l-1) \frac{\log(N)}{\phi}, C l \frac{\log(N)}{\phi} \right]$. For the event $\left\{ \tau_{spr}^{(P)} \geq C l \frac{\log(N)}{\phi} \right\}$ to occur, we need the spreading to be not finished in each of the $l$ intervals. However, at the beginning of each interval, we know that at-least one node is informed of the rumor. Thus, the probability, that the rumor spreading does not complete in a single interval is at-most $N^{-4}$, which follows from monotonicity, where we can bound by saying that exactly one worst-case node is aware of the rumor. As the sequence of callers is independent across intervals, the probability that rumor spreading fails in all $l$ intervals is then at-most $N^{-4l}$. 
\end{proof}

\section{Regret Communication Tradeoff - Proof of Corollary \ref{cor:reg-comm-tradeoff}}
\label{appendix-com-reg-tradeoff}
\begin{proof}
     Consider a fixed $(A_x^{(1)})_{x \in \mathbb{N}}$ and $(A_x^{(2)})_{x \in \mathbb{N}}$, such that $\lim_{x \to \infty} \frac{A_x^{(1)}}{A_{x}^{(2)}} = 0$. 
     The ordering on 
    $ E[(A_{2\tau_{spr}}^{(P)})^{(1)}] \leq E[(A_{2\tau_{spr}}^{(P)})^{(2)}]$ follows 
     trivially as $P$ is fixed for the two cases. It suffices to show that there exist positive constants $N_0$ and $K_0$ (depending on $(A_x^{(1)})_{x \in \mathbb{N}}$ and $(A_x^{(2)})_{x \in \mathbb{N}})$), such that for all $N \geq N_0$ and $K \geq K_0$, $g(A_x^{(1)}) \leq g(A_x^{(2)})$. If $N$ or $K$ is sufficiently large, then
     $(j^{*})^{(i)} = 2(A^{-1})^{(i)}\left( \left(N{K \choose 2} \left(\bigg \lceil \frac{K}{N} \bigg\rceil + 1\right) \right)^{\frac{1}{(2\alpha-6)}} \right)$, for $i \in \{1, 2\}$.
      Notice that 
      \begin{align}
      g((A_x^{(2)})) - g((A_x^{(1)})) &= A_{(j^{*})^{(2)}}^{(2)} - A_{(j^{*})^{(1)}}^{(1)} + \left(  \frac{2}{2\alpha-3}\left(\sum_{l \geq \frac{(j^{*})^{(2)}}{2}-1} \frac{A_{2l+1}^{(2)}}{(A_{l-1}^{(2)})^{3}} \right) -  \frac{2}{2\alpha-3}\left(\sum_{l \geq \frac{(j^{*})^{(1)}}{2}-1} \frac{A_{2l+1}^{(1)}}{(A_{l-1}^{(1)})^{3}} \right) \right), \nonumber
      \\ &\geq A_{(j^{*})^{(2)}}^{(2)} - A_{(j^{*})^{(1)}}^{(1)} - \frac{2}{2\alpha-3}\left(\sum_{l \geq 1} \frac{A_{2l+1}^{(1)}}{(A_{l-1}^{(1)})^{3}} \right) .
      \label{eqn:reg-com-proof}
      \end{align}
      
      Notice that $A_{(j^{*})^{(2)}}^{(2)} - A_{(j^{*})^{(1)}}^{(1)}  > 0$ and scaling (is monotone non-decreasing) with $N$ and $K$. In other words, for fixed $K$, $\lim_{N \to \infty}(A_{(j^{*})^{(2)}}^{(2)} - A_{(j^{*})^{(1)}}^{(1)} ) = \infty$ and for fixed $N$, $\lim_{K \to \infty}(A_{(j^{*})^{(2)}}^{(2)} - A_{(j^{*})^{(1)}}^{(1)} ) = \infty$. This follows as $(A_x^{(i)})_{x \geq 1}$ is super-linear for $i \in \{1,2\}$ and $\lim_{x \to \infty} \frac{A_x^{(1)}}{A_{x}^{(2)}} = 0$. From the hypothesis that the two communication sequences satisfy assumption \textbf{A.2}, we have that 
      $\frac{2}{2\alpha-3}\left(\sum_{l \geq 1} \frac{A_{2l+1}^{(1)}}{(A_{l-1}^{(1)})^{3}} \right) < \infty$ and independent of $N$ and $K$.
       Thus, for all large $N$ or $K$, Equation (\ref{eqn:reg-com-proof}), simplifies to $g(A_x^{(2)}) - g(A_x^{(1)}) > 0$.
\end{proof}

\section{An Algorithm without using agent ids}
\label{app:random_sticky}

The initialization in Line $2$ of Algorithms \ref{algo:main-algo} and \ref{algo:main-algo-prob} relied on each agent knowing its identity. However, in many settings, it may be desirable to have algorithms that do not depend on the agent's identity. We outline here a randomized initialization procedure in Line $2$ to convert Algorithms \ref{algo:main-algo} and \ref{algo:main-algo-prob} to one without using agent ids. Fix some $\gamma \in (0,1)$. We replace Line $2$ in Algorithms \ref{algo:main-algo} and \ref{algo:main-algo-prob} with a randomization, where each agent $i \in [N]$ chooses independently of other agents, a uniformly random subset of size $\bigg \lceil \ln \left( \frac{1}{\gamma} \right) \frac{K}{N} \bigg\rceil + 2$ from the set of $K$ arms as $S_0^{(i)}$. Each agent $i$, then subsequently chooses a random subset of size $\bigg \lceil \ln \left( \frac{1}{\gamma} \right) \frac{K}{N} \bigg\rceil $ uniformly at random from $S_0^{(i)}$ as its `sticky set' $\widehat{S}^{(i)}$. The rest of the algorithms from Line $3$ will be identical. One can then immediately see that the regret guarantees stated in Theorems \ref{thm:strong_result} and \ref{thm:prob_main_result} hold verbatim for this modification, with probability at-least $1-\gamma$, where the probability is over the initial random assignment of the sets $\widehat{S}^{(i)}$ to agents. More precisely, with probability at-least $1-\gamma$, the above random initialization ensures that there exists an agent $i \in [N]$, such that the best arm $1 \in \widehat{S}^{(i)}$. On this event, the regret guarantees along with the same proof of Theorems \ref{thm:strong_result} and \ref{thm:prob_main_result} hold.


\end{appendices}

\end{document}